\documentclass{article}
 \pdfoutput=1
\usepackage[utf8]{inputenc}

\usepackage{microtype}
\usepackage{graphicx}

\usepackage{booktabs} 
\usepackage{caption}
\usepackage{diagbox}
\usepackage{float}
\usepackage{hyperref}
\usepackage{mathtools}

\usepackage[numbers]{natbib}
\usepackage{amsmath}
\usepackage{authblk}
\usepackage{physics}
\usepackage{amsbsy}
\usepackage[utf8]{inputenc}
\usepackage[english]{babel}
\usepackage{algorithm}
\usepackage{algpseudocode}
\usepackage{algorithmicx}
\usepackage{url}            

\usepackage{amsfonts}       
\usepackage{nicefrac}       
\usepackage{microtype}      
\usepackage{enumitem}
\usepackage{amssymb,amsmath,amsthm}
\usepackage{subfig}

\usepackage[margin=1in]{geometry}

\usepackage{xcolor}

\newtheorem{theorem}{Theorem}
\theoremstyle{definition} 
\theoremstyle{lemma} 
\newtheorem{lemma}{Lemma}

\theoremstyle{corollary} 
\newtheorem{corollary}{Corollary}

\theoremstyle{condition} 
\newtheorem{condition}{Condition}
\theoremstyle{assumption} 
\newtheorem{assumption}{Assumption}

\newcommand{\bs}[1]{\boldsymbol{#1}}

\newcommand{\ws}{\bs{w^*}}
\newcommand{\wsgd}{\bs{\bar{w}_{SGD}}}
\newcommand{\wsm}{\bs{\bar{w}_{MKL}}}
\newcommand{\mkl}{MKL-SGD}

\newcommand{\wg}{\bs{w^*}}
\newcommand{\wb}{\bs{w_\mathcal{B}}}

\newcommand{\w}{\bs{w}}
\newcommand{\tw}{\widetilde{\bs{w}}}

\newcommand{\tfw}{\widetilde{F}(\w)}

\title{Choosing the Sample with Lowest Loss makes SGD Robust}
\author[1]{Vatsal Shah}
\author[2]{Xiaoxia Wu}
\author[1]{Sujay Sanghavi}
\affil[1]{Department of Electrical and Computer Engineering, UT Austin}
\affil[2]{Department of Mathematics, UT Austin}

\begin{document}

\maketitle

	\begin{abstract}
The presence of outliers can potentially significantly skew the parameters of machine learning models trained via stochastic gradient descent (SGD). In this paper we propose a simple variant of the simple SGD method: in each step, first choose a set of $k$ samples, then from these choose the one with the smallest current loss, and do an SGD-like update with this chosen sample. Vanilla SGD corresponds to $k=1$, i.e. no choice; $k\geq 2$ represents a new algorithm that is however effectively minimizing a non-convex surrogate loss. Our main contribution is a theoretical analysis of the robustness properties of this idea for ML problems which are sums of convex losses; these are backed up with linear regression and small-scale neural network experiments.\footnote{Appears in AISTATS 2020}
	\end{abstract}
	
\section{Introduction}

This paper focuses on machine learning problems that can be formulated as optimizing the sum of $n$ convex loss functions:
\begin{align}\label{eq:objective}
\min_{\bs{w}}~ F(\w) 
\end{align}
where $F(\w) ~ =\frac{1}{n} \, \sum_{i=1}^n f_i(\bs{w})$ is the sum of convex, continuously differentiable loss functions.

Stochastic gradient descent (SGD) is a popular way to solve such problems when $n$ is large; the simplest SGD update is:
\begin{align} \label{eq:sgdupdatestep}
\text{SGD:} ~ \bs{w_{t+1}} = \bs{w_t} - \eta_t \nabla f_{i_t}(\bs{w_t})
\end{align}
where the sample $i_t$ is typically chosen uniformly at random from [n].



However, as is well known, the performance of SGD and most other stochastic optimization methods is highly sensitive to the quality of the available training data. A small fraction of outliers can cause SGD to converge far away from the true optimum. While there has been a significant amount of work on more robust algorithms for special problem classes (e.g. linear regression, PCA etc.) in this paper our objective is to make a modification to the basic SGD method itself; one that can be easily applied to the many settings where vanilla SGD is already used in the training of machine learning models. 

We call our method Min-$k$ Loss SGD $($\mkl$)$, given below (Algorithm \ref{alg:min2loss}). In each iteration, we first choose a set of $k$ samples and then select the sample with the smallest current loss in that set; the gradient of this sample is then used for the update step. 

\begin{algorithm}[!ht]
	\caption{\mkl}
	\label{alg:min2loss}
	\begin{algorithmic}[1] 
		\State Initialize $\bs{w_0}$
		\State Given samples $D = (\bs{x_t}, y_t)_{t=1}^\infty$
		\For{$t=1, \dots $}
		\State Choose a set $S_t$ of $k$ samples
		\State Select $i_t = \arg \min_{i\in S_t} f_{i}(\bs{w_t})$
		\State Update $\bs{w_{t+1}} =  \bs{w_t} - \eta \nabla f_{i_t}(\bs{w_t})$
		\EndFor
		\State Return $\bs{w_t}$
	\end{algorithmic}
\end{algorithm}

The effectiveness of our algorithm relies on a simple observation:  \emph{in a situation where most samples adhere to a model but a few are outliers skewing the output, the outlier points that contribute the most to the skew are often those with high loss}. In this paper, our focus is on the stochastic setting for standard convex functions.  We show that it provides a certain degree of robustness against outliers/bad training samples that may otherwise skew the estimate. 




\paragraph{Our Contributions}
\begin{itemize}
	\item  To keep the analysis simple yet insightful, we define four natural and {\em deterministic} problem settings - noiseless with no outliers, noiseless with outliers, and noisy with and without outliers - in which we study the performance of \mkl. In each of these settings the individual losses are assumed to be convex, and the overall loss is additionally strongly convex. We are interested in finding the optimum $\ws$ of the ``good" samples, but we do not a-priori know which samples are good and which are outliers.
	\item The expected \mkl~update (over the randomness of sample choice) is {\em not} the gradient of the original loss function (as would have been the case with vanilla SGD); it is instead the gradient of a different non-convex surrogate loss, even for the simplest and friendliest setting of noiseless with no outliers. Our first result establishes that this non-convexity however does not yield any bad local minima or fixed points for \mkl~in this particular setting, ensuring its success.
	\item We next turn to the setting of noiseless with outliers, where the surrogate loss can now potentially have many spurious local minima. We show that by picking a value of $k$ high enough (depending on the condition number of the loss functions that we define) the local minima of \mkl~closest to $\ws$ is better than the (unique) fixed point of SGD.
	\item We establish the convergence rates of \mkl -with and without outliers - for  both the noiseless and noisy settings. 
	\item We back up our theoretical results with both synthetic linear regression experiments that provide insight, as well as encouraging results on the MNIST and CIFAR-10 datasets.
\end{itemize}

\section{Related Work}

The related work can be divided into the following four main subparts:

\paragraph{Stochastic optimization and weighted sampling}
The proposed \mkl algorithm inherently implements a weighted sampling strategy to pick samples. Weighted sampling is one of the popular variants of SGD that can be used for matching one distribution to another (importance sampling), improving the rate of convergence, variance reduction or all of them and has been considered in \cite{kahn1953methods, strohmer2009randomized, zhao2015stochastic, katharopoulos2018not}. Other popular weighted sampling techniques include \cite{needell2014stochastic,moulines2011non, lee2013efficient}. Without the assumption of strong convexity for each $f_i(.)$, the weighted sampling techniques often lead to biased estimators which are difficult to analyze. 
Another idea that is analogous to weighted sampling includes boosting \cite{freund1999short} where harder samples are used to train subsequent classifiers. \emph{However, in presence of outliers and label noise, learning the hard samples may often lead to over-fitting the solution to these bad samples}. This serves as a motivation for picking samples with the lowest loss in \mkl.
\paragraph{Robust linear regression}
Learning with bad training samples is challenging and often intractable even for simple convex optimization problems. For example, OLS is quite susceptible to arbitrary corruptions by even a small fraction of outliers. Least Median Squares (LMS) and least trimmed squares (LTS) estimator proposed in \cite{rousseeuw1984least,vivsek2002least,vivsek2006least} are both sample efficient, have a relatively high break-down point, but require exponential running time to converge.
\cite{huber2011robust} provides a detailed survey on some of these robust estimators for OLS problem. Recently, \cite{bhatia2015robust, bhatia2017consistent, shen2019learning, klivans2018efficient, karmalkar2019list} have proposed robust learning algorithms for linear regression which require the computation of gradient over the entire dataset which may be computationally intractable for large datasets. Another line of recent work considers robustness in the high-dimensional setting (\cite{owen2007robust, xu2009robust, chen2013robust, balakrishnan2017computationally, liu2018high})  In this version, our focus is on general stochastic optimization in presence of outliers.

\paragraph{Robust optimization}
Robust optimization has received a renewed impetus following the works in \cite{diakonikolas2019robust, lai2016agnostic, charikar2017learning}.  %
In most modern machine learning problems, however, simultaneous access to gradients over the entire dataset is time consuming and often, infeasible. \cite{diakonikolas2018sever, prasad2018robust} provides robust meta-algorithms for stochastic optimization under adversarial corruptions. However, both these algorithms require the computation of one or many principal components per epoch which requires atleast $O(p^2)$ computation (\cite{anaraki2014memory}). In contrast, \mkl algorithm runs in $O(k)$ computations per iteration where $k$ is the number of loss evaluations per epoch. In this paper, we don't consider the stronger adversarial model, our focus is on a tractable method that provides robustness on a simpler corruption model (as defined in the next section).

\paragraph{Label noise in deep learning}
\cite{angluin1988learning, kumar2010self, bengio2009curriculum} describe different techniques to learn in presence of label noise and outliers.
\cite{rolnick2017deep} showed that deep neural networks are robust to random label noise especially for datasets like MNIST and CIFAR10. \cite{jiang2017mentornet, ren2018learning} propose optimization methods based on re-weighting samples that often require significant pre-processing. In this paper, our aim is to propose a computationally inexpensive optimization approach that can also provide a certain degree of robustness.

\section{Problem Setup}

We make the following assumptions about our problem setting described in \ref{eq:objective}. Let $\mathbb{O}$ be the set of outlier samples; this set is of course unknown to the algorithm. We denote the optimum of the non-outlier samples by $\ws$, i.e.
$$ \ws ~ := ~ \arg \min_{\bs{w}} \, \sum_{i \notin \mathbb{O}} f_i(\bs{w})$$
In this paper we show that \mkl~allows us to estimate $\ws$ without a-priori knowledge of the set $\mathbb{O}$, under certain conditions. We now spell these conditions out. 

\begin{assumption}[\textbf{Individual losses}]
	Each $f_i(\bs{w})$ is convex in $\bs{w}$, with Lipschitz continuous gradients with constant $L_i$.
	$$\norm{\nabla f_i(\bs{w_1}) -\nabla f_i(\bs{w_2})} \leq L_i \norm{ \bs{w_1} - \bs{w_2}} $$ 
	and define $L :=max_i L_i$
\end{assumption}{}
\noindent
It is common to also assume strong convexity of the overall loss function $F(\cdot)$. Here, since we are dropping samples, we need a very slightly stronger assumption.
\begin{assumption}[\textbf{Overall loss}]
	For any $n-k$ size subset $S$ of the samples, we assume the loss function $\sum_{i\in S} f_i(\bs{w})$ is {\em strongly convex} in $\bs{w}$. Recall that here $k$ is the size of the sample set in the \mkl~algorithm.
\end{assumption}
\noindent
Lastly, we also assume that all the functions share the same minimum value. Assumption 3 is often satisfied by most standard loss functions with a finite unique minima \cite{gunasekar2018characterizing} such as squared loss, hinge loss, etc.
\begin{assumption}[\textbf{Equal minimum values}]
	Each of the functions $f_i(.)$ shares the same minimum value $\min_{\bs{w}} f_i(\bs{w}) = \min_{\bs{w}} f_j(\bs{w})~\forall~ i,j$. 
\end{assumption}

We are now in a position to formally define three problem settings we will consider in this paper. For each $i$ let $C_i := \{\hat{\bs{w}} : \hat{\bs{w}} = \arg \min_{\bs{w}} f_i(\bs{w}) \}$ denote the set of optimal solutions (there can be more than one because $f_i(\cdot)$ is only convex but not strongly convex). Let $d(a, S)$ denote the shortest distance between point $a$ and set $S$.

\noindent
\paragraph{Noiseless setting with no outliers:} As a first step and sanity check, we consider what happens in the easiest case: where there are no outliers. There is also no ``noise", by which we mean that the optimum $\ws$ we seek is also in the optimal set of every one of the individual sample losses, i.e.
$$\ws \in C_i ~ \text{for all $i$}$$
Of course in this case vanilla SGD (and many other methods) will converge to $\ws$ as well; we just study this setting as a first step and also to build insight.
\noindent
\paragraph{Outlier setting:} Finally, we consider the case where a subset $\mathbb{O}$ of the samples are outliers. Specifically, we assume that for outlier samples the $\ws$ we seek lies far from their optimal sets, while for the others it is in the optimal sets:
$$ d(\ws,C_i) \geq 2\delta ~ \text{for all $i\in \mathbb{O}$}$$
$$ \ws \in C_i ~ \text{for all $i\notin \mathbb{O}$} $$
Note that now vanilla SGD on the entire loss function will {\em not} converge to $\ws$. 
\noindent
\paragraph{Noisy setting:} As a second step, we consider the case when samples are noisy but there are no outliers. In particular, we model noise by allowing $\ws$ to now be outside of individual optimal sets $C_i$, but not too far; specifically, \\
\textit{No outliers} $$ d(\ws,C_i) \leq \delta ~ \text{for all $i$}$$
\textit{With outliers} 
$$ d(\ws,C_i) \leq \delta ~ \text{for all $i\notin \mathbb{O}$}$$
$$ d(\ws,C_i) > 2\delta ~ \text{for all $i\in \mathbb{O}$}$$

For the noisy setting, we will focus only on the convergence guarantees. We will show that \mkl~gets close to $\ws$ in this setting; again in this case vanilla SGD will do so as well for the no outliers setting of course.



\section{Understanding \mkl}

To build some intuition for \mkl, we describe the notation and look at some simple settings. Recall \mkl~takes $k$ samples and then retains the one with lowest current loss; this means it is sampling non-uniformly. For any $\bs{w}$, let $m_1(\bs{w}), m_2(\bs{w}), m_3(\bs{w}), \dots m_n(\bs{w})$ be the sorted order w.r.t. the loss at that $\bs{w}$, i.e.
\begin{align*}
f_{m_1(\bs{w})}(\bs{w}) \leq 	f_{m_2(\bs{w})}(\bs{w}) \leq \dots \leq 	f_{m_n(\bs{w})}(\bs{w}) 
\end{align*}
Recall that for a sample to be the one picked by \mkl~for updating $\bs{w}$, it needs to first be part of the set of $k$ samples, and then have the lowest loss among them. 
A simple calculation shows that probability that the $i^{th}$ best sample $m_i(\bs{w})$ is the one picked by \mkl~is given by
\begin{align}
p_{m_i(\bs{w})}(\bs{w}) = \begin{cases}{}
&\dfrac{{n-i \choose k-1}}{{n\choose k}}\qquad \qquad \qquad \qquad\qquad \qquad\text{without replacement}\\
&\dfrac{(n - (i-1))^k - (n-i)^k}{n^k}
\qquad \qquad \text{with replacement}
\end{cases} \label{eq:probability}
\end{align}
In the rest of the paper, we will focus on the ``with replacement" scenario for ease of presentation; this choice does not change our main ideas or results. With this notation, we can rewrite the expected update step of \mkl~ as
$$
\mathbb{E}[\w_+ | \w] = \w - \eta \sum_i p_{m_i(\bs{w})} \nabla f_{m_i(\bs{w})}(\bs{w}) 
$$ 
For simplicity of notation in the rest of the paper, we will relabel the update term in the above by defining as follows:
$$\nabla \tfw  ~ := ~ \sum_i p_{m_i(\bs{w})} \nabla f_{m_i(\bs{w})}(\bs{w})$$
Underlying this notation is the idea that, in expectation, \mkl~is akin to gradient descent on a {\em surrogate} loss function $\tilde{F}(\cdot)$ which is different from the original loss function $F(\cdot)$; indeed if needed this surrogate loss can be found (upto a constant shift) from the above gradient. We will not do that explicitly here, but instead note that even with all our assumptions, indeed even without any outliers or noise, this surrogate loss can be  non-convex. It is thus important to see that \mkl~does the right thing in all of our settings, which is what we describe next.

\subsection{Noiseless setting with no outliers}

As a first step (and for the purposes of sanity check), we look at \mkl~in the simplest setting when there are no outliers and no noise. Recall from above that this means that $\ws$ is in the optimal set of every single individual loss $f_i(\cdot)$. However as mentioned above, even in this case the surrogate loss can be non-convex, as seen e.g. in Figure \ref{fig:3dsurfaceplot} for a simple example.
\begin{figure}[!ht]
	\centering
	\includegraphics[scale=0.2]{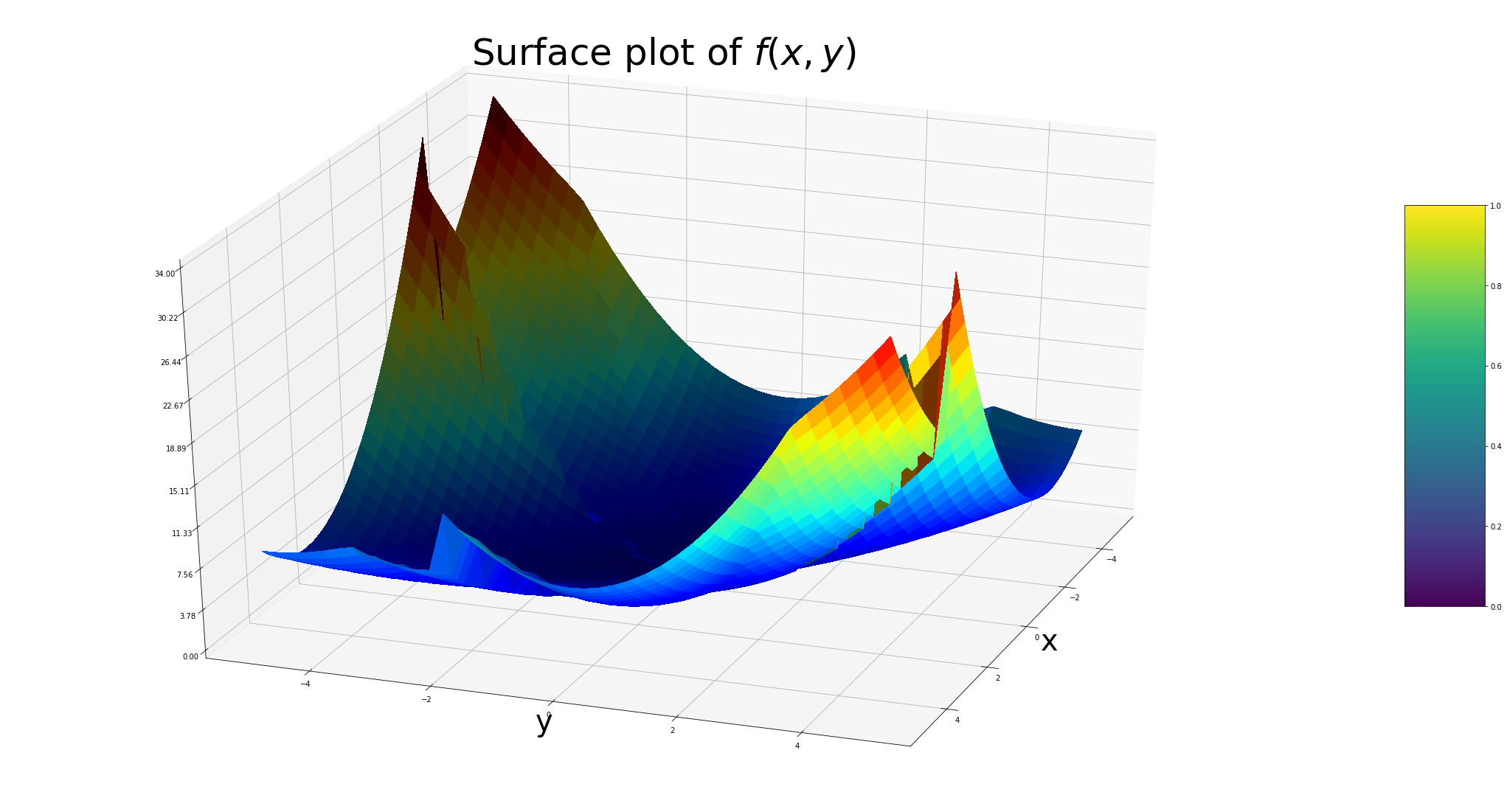}
	\caption{Non-convexity of the surface plot with three samples in the two-dimensional noiseless linear regression setting}
	\label{fig:3dsurfaceplot}
\end{figure}
However, in the following lemma we show that even though the overall surrogate loss $\tilde{F}(\cdot)$ is non-convex, in this no-noise no-outlier setting it has a special property with regards to the point $\ws$.
\begin{lemma} \label{lemma:sc}
	In the noiseless setting, for any $\w$ there exists a $\lambda_{\w} > 0$ such that $$\nabla \tfw^\top (\bs{w} - \ws)  \geq \lambda_{\w} \norm{\w - \ws}^2.$$
\end{lemma}
In other words, what this lemma says is that on the line between any point $\bs{w}$ and the point $\ws$, the surrogate loss function $\tilde{F}$ is convex from any point -- even though it is not convex overall. This is akin to the restricted secant inequality condition described in \cite{karimi2016linear, zhang2017restricted}. The following theorem uses this lemma to establish our first result: that in the noiseless setting with no outliers, $\ws$ is the only fixed point (in expectation) of \mkl.





\begin{theorem}[\textbf{Unique stationary point}] \label{theorem:nolocalminima}
	For the noiseless setting with no outliers, and under assumptions $1-3$, the expected \mkl~update satisfies $\nabla \tfw = 0$ if and only if $\w = \ws$.
\end{theorem}

\subsection{Noiseless setting with Outliers}

In presence of outliers, the surrogate loss can have multiple local minima that are far from $\ws$ and indeed potentially even worse than what we could have obtained with vanilla SGD on the original loss function. We now analyze \mkl~in the simple setting of symmetric squared loss functions and try to gain useful insights into the landscape of loss function for the scalar setting. We would like to point out that the analysis in the next part serves as a clean template and can be extended for many other standard loss functions used in convex optimization. 

\paragraph{Squared loss in the scalar setting}
Figure \ref{fig:scalar} will be a handy tool for visualizing and understanding both the notation and results of this subsection.
Consider the case where all losses are squared losses, with all the clean samples centered at $w^*$ and all the outliers at $w_B$, but all having different Lipschitz constants. Specifically, consider:

\begin{align} \label{eq:scalarloss}
f_i(w) = \begin{cases} &l_i (w -w^*)^2 ~~\forall~i \notin \mathbb{O} \\
&l_i (w -w_B)^2 ~~ \forall~i \in \mathbb{O}, 
\end{cases}
\end{align} 

Let $l_{m} := \min_{i\notin\mathbb{O}} l_i$ and Let $l_{M} := \max_{i\in\mathbb{O}} l_i$ and $l_{max} = \max_{i\in[n]} l_i$, $l_{min} = \min_{i \in [n]}l_i$. Let us define $\kappa=\dfrac{l_{max}}{l_{min}} \geq \dfrac{l_{M}}{l_{m}}$.
We initialize \mkl~at $w_0 = w_B$, a point where the losses of outlier samples are 0 and all the clean samples have non-zero losses. As a result at $w_B$, \mkl~ has a tendency to pick all the outlier samples with a higher probability than any of the clean samples. This does not bode well for the algorithm since this implies that the final stationary point will be heavily influenced by outliers.
Let $\bar{w}_{MKL}$ be the stationary point of \mkl~for this scalar case when initialized at $w_B$. 

Let us define $\widetilde{w}$ as follows: 
\begin{align} \label{eq:tildewwww}
\widetilde{w} :=\begin{cases}\begin{drcases} w \mid &w = \min_{\alpha} \alpha w^* + (1-\alpha)w_B, \alpha\in (0, 1),\quad f_{l_{m}}(w) = f_{l_{M}}(w) \end{drcases}\end{cases} 
\end{align}{} 
Thus, $\widetilde{w}$ is the closest point to $w_B$ on the line joining $w_B$ and $w^*$ where the loss function of one of the clean samples and one of the outliers intersect as illustrated in Figure \ref{fig:scalar}. 

By observation, we know for the above scalar case  $\widetilde{w}= \dfrac{\sqrt{l_m} w^* + \sqrt{l_M} w_B }{\sqrt{l_m} + \sqrt{l_M}}$. Let $\hat{p}(\w_0) = \sum_{j \in \mathbb{O}} p_j(\w_0)$ represent the total probability of picking outliers at the starting point $\w_0$. The maximum possible value that can be attained $\hat{p}$ over the entire landscape is given as:
\begin{align} \label{eq:pmax}
\hat{p}_{max}=\max_{w} \hat{p}(\w) = \sum_{i=1}^{\mathbb{|O|}} p_{m_i(\w)}(\w)
\end{align}{}  

The next condition gives a sufficient condition to avoid all the bad local minima are avoided no matter where we initialize. For the simple scalar case, the condition is:
\begin{condition} \label{cond:1}
	$\hat{p}_{max} < \dfrac{1}{1+\kappa\sqrt{\kappa}}$
\end{condition}{}

\begin{figure}
	\centering
	\includegraphics[scale=0.5]{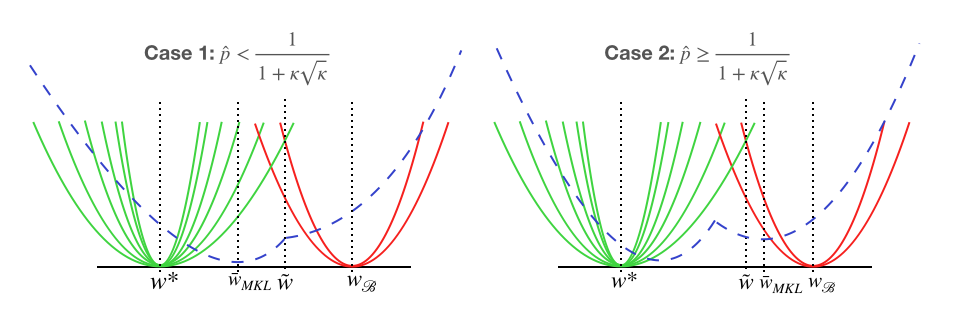}
	\caption{Illustration with conditions when bad local minima will or will not exist. Here, we demonstrate that even if we start at an initialization $\wb$ that assigns the highest probabilities to bad samples (red),  it is possible to avoid the existence of a bad local minima if Condition 1 is satisfied. Recursively, we show in Lemma \ref{lemma:vector} that it is possible to avoid all bad local minima and reach a good local minima (where the good samples have the highest probabilities)}
	\label{fig:scalar}
\end{figure}{}

To further elaborate on this, for the loss functions and $\widetilde{w}$ defined in equations \eqref{eq:scalarloss} and \eqref{eq:tildewwww} respectively, if condition 1 is not satisfied, then we cannot say anything about where \mkl~ converges. 
However, if condition 1 holds true, then we are in Case 1 (Figure \ref{fig:scalar}), i.e.
the stationary point attained by \mkl~will be such that it is possible to avoid the existence of the first bad local minima. The first bad local minima occurs by solving the optimization problem where the top-$|\mathbb{O}|$ highest probabilities are assigned to the bad samples. 

Following the above analysis recursively, we can show that all other subsequent bad local minimas are avoided as well, until we reach the local minima which assigns the largest $(n-|\mathbb{O}|)$ probabilities to the clean samples\footnote{Refer to Appendix section \ref{sec:vector} for further details on this discussion}. This indicates that irrespective of where we initialize in the $1D$ landscape, we are bound to end up at a local minima with the highest probabilities assigned to the clean samples. In the latter part of this section, we will show that \mkl~solution attained when Case 1 holds is provably better than the SGD solution.
However, if condition 1 is false (Case 2, Figure \ref{fig:scalar}), then it is possible that \mkl~gets stuck at any one of the many local minimas that exist close to the outlier center $w_B$ and we cannot say anything about the relative distance from $\ws$.

A key takeaway from the above condition is that \emph{for a fixed $n$ as $\kappa$ increases, we can tolerate smaller $\hat{p}$ and consequently smaller fraction of corruptions $\epsilon$}. For a fixed $\epsilon$ and $n$, increasing the parameter $k$ (upto $k < \frac{n}{2}$) in \mkl~leads to an increase in $\hat{p}$ and thus increasing $k$ can lead to the violation of the above condition. 
This happens because samples with lower loss will be picked with increasing probability as $k$ increases and as a result the propensity of \mkl~to converge towards the closest stationary point it encounters is higher.

\paragraph{Squared loss in the vector setting}
The loss functions are redefined as follows: 
\begin{align} \label{eq:vectorloss}
&f_i(\w) = \begin{cases}{}
&l_i \norm{\w -\ws}^2 ~~\forall~i \notin \mathbb{O}  \\
&l_i \norm{\w -\w_{b_i}}^2 ~~ \forall~i \in \mathbb{O},  
\end{cases}
\end{align}

Without loss of generality, assume that $2\delta<\norm{\w_{b_1}-\ws} \leq \norm{\w_{b_2}-\ws}\leq \dots \leq \norm{\w_{b_{|\mathbb{O}|}}-\ws}  $ and $\gamma = \dfrac{2\delta}{\norm{\w_{b_{|\mathbb{O}|}}-\ws}}$. Let $\bar{\w}$ be any stationary attained by \mkl.
Suppose $\theta_{i,\bar{\w}}$ be the angle between the line passing through $\w_{b_i}$ and $\ws$ and the line connecting $\bar{\w}$ and $\ws$. Let us define $\theta_{M,\bar{\w}} := \max_i\theta_{i,\bar{\w}}$
Consider $\kappa=\dfrac{\max_{i\in[n]} l_i}{\min_{i\in[n]} l_i}$. Let $\hat{p}(\w_0) = \sum_{j \in \mathbb{O}} p_j(\w_0)$ represent the total probability of picking outliers at the starting point $\w_0$. The maximum possible value that can be attained $\hat{p}$ is given as:
\begin{align} \label{eq:pmax}
\hat{p}_{max}=\max_{w} \hat{p}(\w) = \sum_{i=1}^{\mathbb{|O|}} p_{m_i(\w)}(\w)
\end{align}{} 
where for any $\w$,  $p_{m_i(\w)}(\w)$ are ordered i.e. $p_{m_1(\w)}(\w)> p_{m_2(\w)}(\w)>\dots >p_{m_n(\w)}(\w) $.

At $\ws$, by definition, we know that $\forall~i \notin \mathbb{O}$, $f_i(\ws)= 0$ and $\forall~j\in \mathbb{O}$, $f_j(\ws) > 0$. By continuity arguments, there exists a ball of radius $r>0$ around $\ws$, $\mathcal{B}_r(\ws)$, defined as follows:
\begin{align} \label{eq:defball}
\mathcal{B}_r(\ws) = \begin{cases}\begin{drcases}{} \w \mid  &f_i(\bs{w}) < f_j(\bs{w})~ \forall~i \notin \mathbb{O},~j \in \mathbb{O}, \\
& \norm{\w -\ws}\leq r \hspace{1cm}\end{drcases} \end{cases}
\end{align}{}
In the subsequent lemma, we show that that it is possible to drift into the ball $\mathcal{B}_r(\ws)$ where the clean samples have the highest probability or the lowest loss. \footnote{It is trivial to show the existence of a ball of radius $r>0$ for any set of continuously differentiable $f_i(.)$.}

\begin{lemma} \label{lemma:vector}
	Consider the loss function and $\mathcal{B}_r(\ws)$ as defined in equations \eqref{eq:vectorloss} and \eqref{eq:defball} respectively. Suppose $q = \dfrac{\cos{\theta_{M,\bar{\w}}}}{\gamma} - 1 +\dfrac{\sqrt{\kappa}\cos{\theta_{M,\bar{\w}}}}{\gamma}>0$ and $\hat{p}_{max}$ as defined in Equation \eqref{eq:pmax} satisfies $\hat{p}_{max} \leq \dfrac{1}{1 + \kappa q}$. Starting from any initialization $\w_0$, for any stationary point $\bar{\w}$ attained by \mkl, we have that $\bar{\w} \in \mathcal{B}_r(\ws)$
\end{lemma}{}
In other words, initializing at any point in the landscape, the final stationary point attained by \mkl~will inevitably assign the largest $n-|\mathbb{O}|$ probabilities to the clean samples.
The proof is availabe in Appendix Section \ref{sec:vector}.
For the scalar case, $d=1$, we have $\theta_{j,\bar{\w}} = 0~\forall~j$. If $\gamma=1$ and all the outliers are centered at the same point, then in the scalar setting the condition in Lemma \ref{lemma:vector} reduces to condition 1.

Note that, the above lemma leads to a very strong worst-case guarantee. It states that the farthest optimum will always be within a bowl of distance $r$ from $\ws$ no matter where we initialize. Moreover, as long as the condition is satisfied no matter where the outliers lie (can be adversarially chosen), \mkl~ always has the propensity to bring the iterates to a ball of radius $r$ around $\ws$. However, when the necessary conditions for its convergence are violated, the guarantees are initialization dependent. Thus, all the discussions in the rest of this section will be with respect to these worst case guarantees. However, as we see in the experimental section for both neural networks and linear regression, random initialization also seems to perform better than SGD.

\paragraph{Effect of $\kappa$}
A direct result of Lemma \ref{lemma:vector} is that higher the condition number of the set of quadratic loss functions, lower is the fraction $\epsilon$ of outliers the \mkl~can tolerate. This is because large $\kappa$ results in a small value of $\dfrac{1}{1 + \kappa q}$. This implies that $\hat{p}$ has to be small which in turn requires smaller fractions fo corruptions, $\epsilon$. 

\paragraph{Effect of $\gamma$:}  
The relative distance of the outliers from $\ws$ plays a critical role in the condition for Lemma \ref{lemma:vector}. We know that $\gamma \in (0,1]$.  $\gamma =1$ implies the outliers are equidistant from the optimum $\ws$. 
Low values of $\gamma$ lead to a large $q$ leading to the violation of the condition with $\hat{p}$ (since RHS in the condition is very small), which implies that one bad outlier can guarantee that the condition in Lemma \ref{lemma:vector} are violated. The guarantees in the above lemma are only when the outliers are not adversarially chosen to lie at very high relative distances from $\ws$. One way to avoid the set of outliers far far away from the optimum is to have a filtering step at the start of the algorithm like the one in \cite{diakonikolas2018sever}. We will refer this in Experiments.

\paragraph{Effect of $\cos{\theta_{j,\bar{\w}}}$:} At first glance, it may seem that $\cos{\theta_{j,\bar{\w}}}=0$ may cause $1+\kappa q < 0$ and since $\hat{p}(\w) >0 $, the condition in Lemma \ref{lemma:vector} may never be satisfied. Since, the term $\cos{\theta_{j,\bar{\w}}}$ shows up in the denominator of the loss associated with outlier centered at $\w_{b_j}$. Thus, low values of $\cos{\theta_{j,\bar{\w}}}$ implies high value of loss associated with the function centered at $\w_{b_j}$ which in turn implies the maximum probability attained by that sample can never be in the top-$|\mathbb{O}|$ probabilities for that $\bar{\w}$.

\paragraph{Analysis for the general outlier setting:}

In this part, we analyze the fixed point equations associated with \mkl~and SGD and try to understand the behavior \emph{in a ball $\mathcal{B}_r(\ws)$ around the optimum?} For the sake of simplicity, we will assume that $\norm{\nabla f_i(\bs{w})} \leq G ~\forall~i\in \mathbb{O}$. Next, we analyze the following two quantities: i) distance of $\wsgd$ from $\ws$ and distance of the any of the solutions attained by $\wsm$ from $\ws$.

\begin{lemma}\label{lemma:sgdbound}
	Let $\wsgd$ indicate the solution attained SGD.
	Under assumptions 1-3, there exists an $\epsilon'$ such that for all $\epsilon \leq \epsilon'$, 
	\begin{align*}
	\epsilon G \leq (1 - \epsilon)  L \norm{\wsgd - \ws} 
	\end{align*}
\end{lemma}
Using Lemma \ref{lemma:sc}, we will define $\lambda$ as follows:
\begin{align}\label{eq:lambdasc}
\lambda:= \min_{\w} \lambda_{\w}
\end{align}{}
In Appendix Section \ref{sec:proofnolocalminima}, we show that $\lambda>0$, however the exact lower bounds for this $\lambda$ are loss function dependent. Naively, $\lambda_{\w} = \min_i p_i(\w) \lambda$
\begin{lemma} \label{lemma:m2lboundstrongconvexity}
	Let $\wsm$ be any first order stationary point attained by \mkl.
	Under assumptions 1-3, for a given $\epsilon <1$ and $\lambda$ as defined in equation \eqref{eq:lambdasc}, there exists a $k'$ such that for all $k \geq k'$,
	\begin{align*}
	\norm{\wsm -\ws}  \leq \dfrac{\epsilon^k G}{\lambda}
	\end{align*}
\end{lemma}
Finally, we show that any solution attained by \mkl~is provably better than the solution attained by SGD. We would like to emphasize that this is a very strong result. The \mkl~has numerous local minima and here we show that even the worst\footnote{farthest solution from $\ws$} solution attained by \mkl~is closer to $\ws$ than the solution attained by SGD. Let us define
$\alpha(\epsilon, L, k, \lambda) = \dfrac{(1 - \epsilon)  L \epsilon^{k-1} }{\lambda} $
\begin{theorem} \label{theorem:relativebounds}
	Let $\wsgd$ and $\wsm$ be the the stationary points attained by SGD and \mkl~algorithms respectively for the noiseless setting with outliers. Under assumptions 1-3, for any $\wsm \in \mathcal{B}_r(\ws)$ and $\lambda$ defined in equation \eqref{eq:lambdasc}, there exists an $\epsilon'$ and $k'$ such that for all $\epsilon \leq \epsilon'$ and $k\geq k'$, we have $\alpha(\epsilon, L, k, \lambda)<1$ and, 
	\begin{align}
	\norm{ \wsm - \ws} < \alpha(\epsilon, L, k, \lambda)  \norm{ \wsgd - \ws} 
	\end{align}
\end{theorem}
For squared loss in scalar setting, we claimed that for a fixed $n$ and $\epsilon$, using a large $k$ may not be a good idea. Here, however once we are in the ball, $\mathcal{B}_r(\ws)$, using larger $k$ (any $k < \dfrac{n}{2}$), reduces $\alpha(\epsilon, L, k, \lambda)$ and allows \mkl~to get closer to $\ws$. 

The conditions required in Lemma \ref{lemma:vector} and Theorem \ref{theorem:relativebounds} enable us to provide guarantees for only a subset of relatively well-conditioned problems. We would like to emphasize that the bounds we obtain are worst case bounds and not in expectation. As we will note in the Section \ref{sec:experiments} and the Appendix, however these bounds may not be necessary, for standard convex optimization problems \mkl~easily outperforms SGD. 

\section{Convergence Rates}\label{sec:convergence}
In this section, we go back to the in expectation convergence analysis which is standard for the stochastic settings. 
For smooth functions with strong convexity, \citep{moulines2011non,needell2014stochastic} provided guarantees for linear rate of convergence. We restate the theorem here and show that the theorem still holds for the non-convex landscape obtained by \mkl~in noiseless setting.  
\begin{lemma}[\textbf{Linear Convergence} \citep{needell2014stochastic}] \label{lemma:rachel}
	Let $F(\w) = \mathbb{E}[f_i(\bs{\w})]$ be $\lambda$-strongly convex. Set $\sigma^2= \mathbb{E}[\|\nabla f_i(\ws)\|^2]$ with $\bs{w}^*:=argmin F(\bs{w})$. Suppose  $\eta\leq \dfrac{1}{\sup_i L_i}$. Let $\bs{\Delta}_t =\bs{w_{t}} - \ws$. After  $T$ iterations, SGD satisfies:
	\begin{align}
	\mathbb{E}\left[\|\bs{\Delta}_T\|^2\right]\leq (1-2\eta \hat{C})^T \|\bs{\Delta}_0 \|^2+\eta R_\sigma \label{eq:convergence}
	\end{align}
	where $\hat{C}= \lambda (1-\eta \sup_i L_i)$ and $R_\sigma=  \dfrac{ \sigma^2}{\hat{C}}$.
\end{lemma}
In the noiseless setting, we have $ \|\nabla f_i(\ws)\|=0$ and so $\sigma:=0$. $\bs{w}^*$ in \eqref{eq:convergence} is the same as $\bs{w}^*$  stated in Theorem \ref{theorem:nolocalminima}. Even though above theorem is for SGD, it still can be applied to our algorithm \ref{alg:min2loss}. At each iteration there exists a parameter $\lambda_{\bs{w_t}}$  that could be seen as the strong convexity parameter (c.f. Lemma \ref{lemma:sc}). For MKL-SGD, the parameter $\lambda$ in \eqref{eq:convergence} should be  $\lambda = \min_{t}\lambda_{\bs{w}_t}$. Thus, MKL-SGD algorithm still guarantees \textit{linear convergence} result but with an implication of slower speed of convergence than standard SGD.  

However, Lemma \ref{lemma:rachel}  will not hold for \mkl~in noisy setting since there exists no strong convexity parameter.  Even for noiseless setting, the rate of convergence for MKL-SGD given in Lemma \ref{lemma:rachel} is not tight.  The upper bound in \eqref{eq:convergence} is loosely set to the constant $\lambda:=\min_t\lambda_{\bs{w}_t}$ for all the iterations. We fix it by concretely looking at each iteration. We give a general bound for the any stochastic algorithm (c.f. Theorem \ref{thm:R}) for both noiseless and noisy setting in absence and presence of outliers. 
\begin{theorem}[\textbf{Distance to $\ws$}]  \label{thm:R} 
	Let $\bs{\Delta}_t =\bs{w_{t}} - \ws$. Denote the strong convexity parameter $\lambda_{good}$  for all the good samples. 
	Let 
	$$ \psi = 2\eta_t  \lambda_{good }(1-\eta_t \sup_iL_i)\min_{ i\notin \mathbb{O}}p_i(\bs{w_t})$$
	Suppose at $t^{th}$ iteration, the stepsize is set as $\eta_t$, then conditioned on the current parameter $\bs{w}_t$, the expectation of the distance between the $\bs{w}_{t+1}$ and $\ws$ can be upper bounded as:
	\begin{align}
	\mathbb{E}_i\left[ \norm{\bs{\Delta}_{t+1}}^2|\bs{w_{t}} \right]
	\leq & \left( 1-  \psi \right)\|\bs{\Delta}_{t}\|^2+ \eta_t R_t \label{eq:residual}
	\end{align}
	where 
	\begin{align*}
	R_t 
	=	& - 2  \sum_{i\notin \mathbb{O}} p_i(\bs{w_t}) \langle{ \bs{w_t} - \wg,\nabla f_{i}( \wg)\rangle}+ \sum_{i\in \mathbb{O}}p_i(\bs{w_t}) \left(2\eta_t\norm{\nabla f_{i}( \wg)}^2+ \eta_t\norm{\nabla f_{i}( \bs{w_{t}})}^2+2\left(f_{i}(\wg)- f_i(\bs{w_t} )\right) \right)
	\end{align*}
\end{theorem}

Theorem \ref{thm:R} implies that for any stochastic algorithm in the both noisy and noiseless setting, outliers can make the upper bound ($R_t$) much worse as it produces an extra term (the third term in $R_t$). \emph{The third term in $R_t$ has a lower bound that could be an increasing function of $|\mathbb{O}|$}. However, its impact can be reduced by appropriately setting $p_i(\bs{w_t})$, for instance using a larger $k$ in \mkl.  In the appendix, we also provide a sufficient condition (Corollary 1 in the Appendix) when \mkl~is always better than standard SGD (in terms of its distance from $\ws$ in expectation).  

The convergence rate depends on the constant $\psi \propto \min_{ i\notin \mathbb{O}} p_i(\bs{w_t}) $. Note that this term $ \min_{ i\notin \mathbb{O}} p_i(\bs{w_t}) $ is not too small for our algorithm MKL-SGD since it is a minimum among all the \emph{good}   sample (not including the outliers).  However, when compared with vanilla SGD where $ \min_{ i\notin \mathbb{O}} p_i(\bs{w_t}) =1/N$, $\min_{ i\notin \mathbb{O}} p_i(\bs{w_t})$ with $p_i(\bs{w_t})$ defined in \eqref{eq:probability} for MKL-SGD, in some sense, could be smaller than $1/N$. For instance, in the experiments given in Figure 5 (a)-(c) (Appendix \ref{sec:linReg}), the slope of SGD is steeper than MKL-SGD, which implies that $\psi^{(MKL)} < \psi^{(SGD)} $.

To understand the residual term $R_t$. Let us take the noiseless setting with outliers for an example.  We have $\nabla f_i(\ws) = 0$ and  $f_i(\ws)=0$ for all $i\notin \mathbb{O}$. But for  $i\in \mathbb{O}$, $\nabla f_i(\ws) \neq 0$ and  $f_i(\ws)\neq 0$.  Then the term $R_t$ can be reduced to
\begin{align}
R_t = &  \sum_{i\in \mathbb{O}}p_i(\bs{w_t}) \left(2\eta_t\norm{\nabla f_{i}( \wg)}^2+ \eta_t\norm{\nabla f_{i}( \bs{w_{t}})}^2+2\left(f_{i}(\wg)- f_i(\bs{w_t} )\right) \right) \label{eq:expressR}
\end{align} If  we are at the same point $\bs{w_t}$ for both SGD and MKL-SGD and $p_i(\bs{w_t})< \dfrac{1}{N}$ for $i \in \mathbb{O}$, we have
$R_t^{(SGD)}>R_t^{(MKL)}$. It means that  MKL-SGD could reach to a neighbor with a radius that is possibly smaller than vanilla SGD algorithm,  with a rate proportional to $ \min_{ i\notin \mathbb{O}} p_i(\bs{w_t}) $ but not necessarily faster than vanilla SGD.

\begin{figure}[!htbp] 
	\centering
	\includegraphics[scale=0.32]{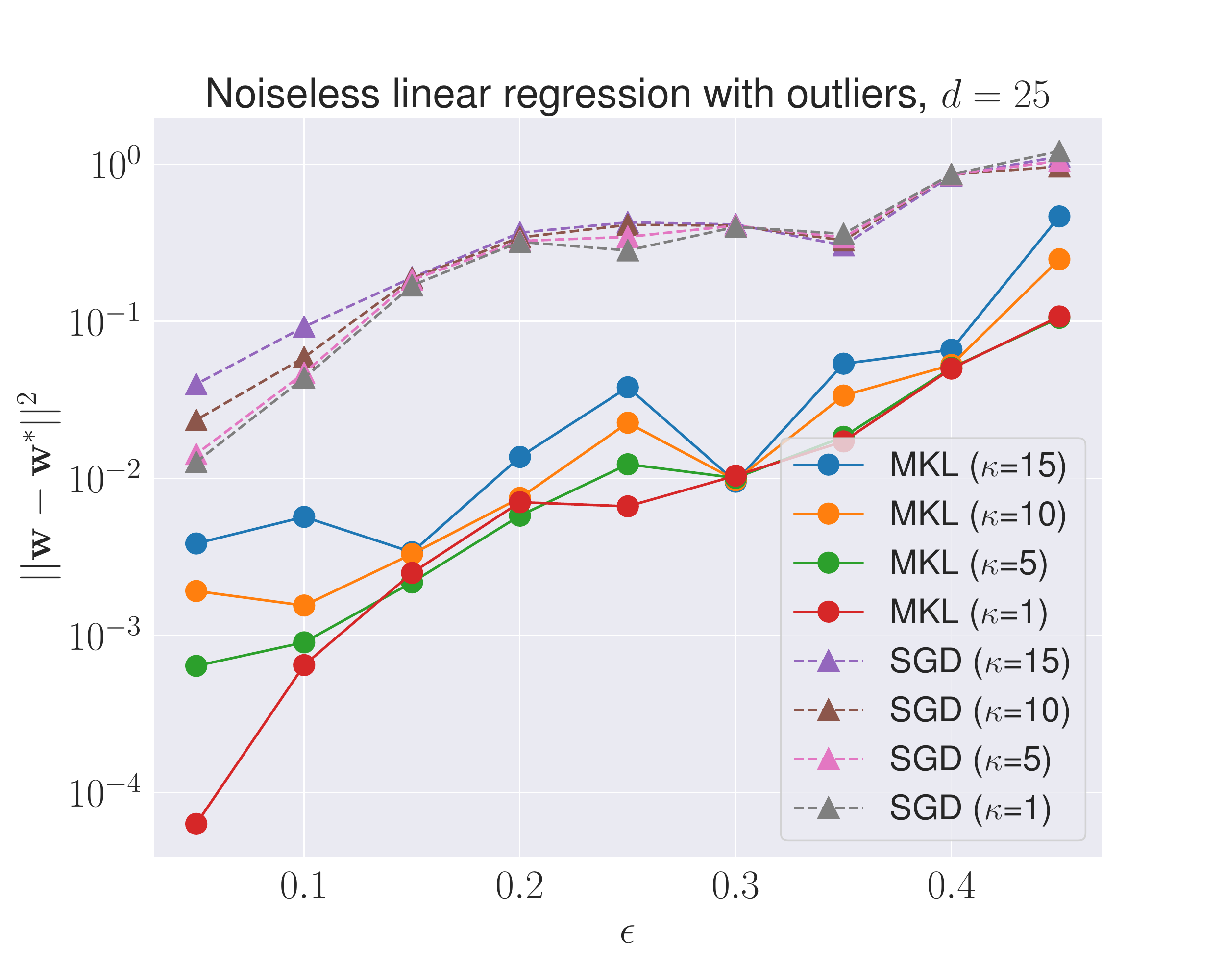}
	\includegraphics[scale=0.32]{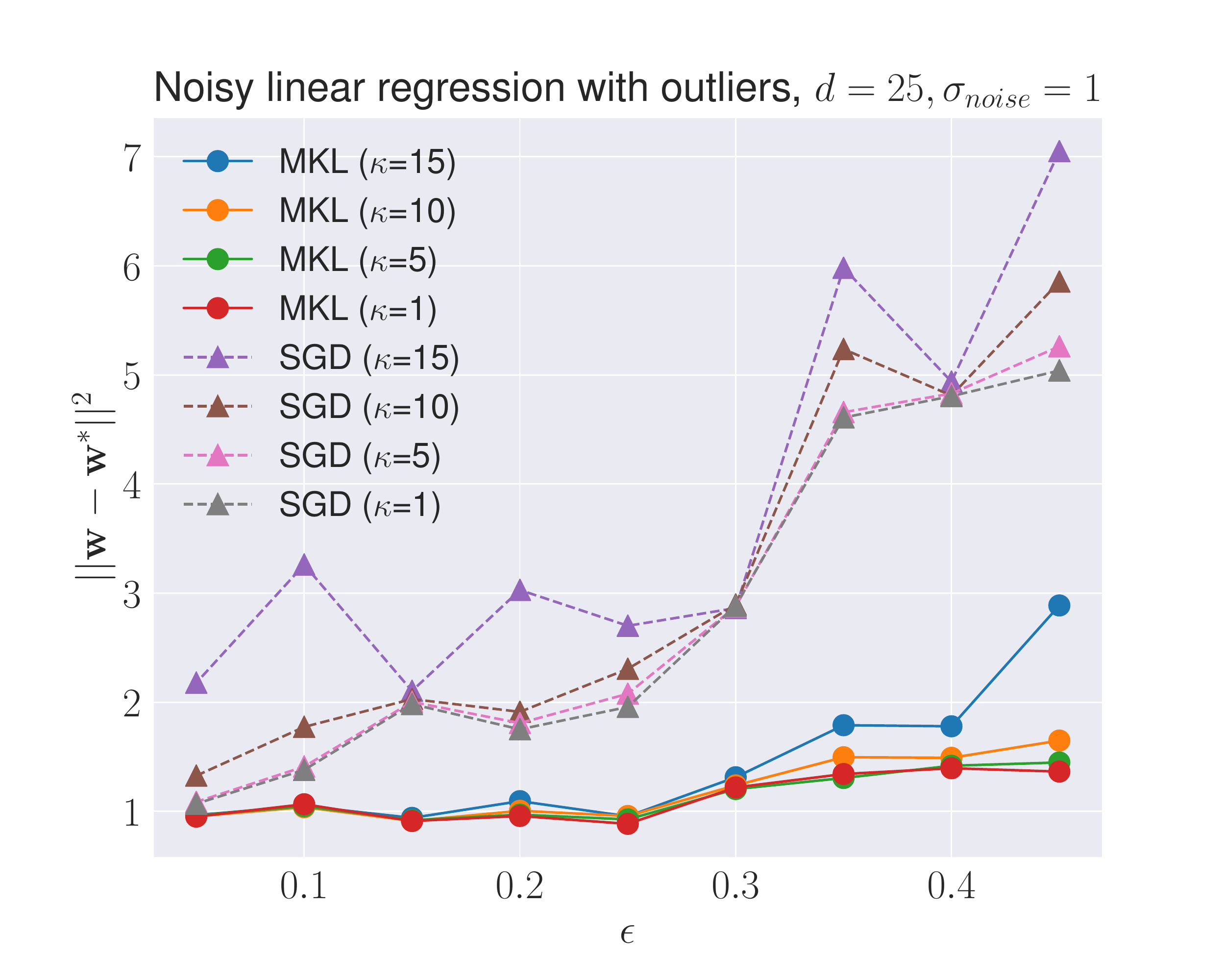}
	\caption{Comparing the performance of MKL-SGD ($k=2$) and SGD for different values of $\kappa$ in noiseless and noisy linear regression against varying fraction of outliers.}\label{fig:f2}
\end{figure}
\section{Experiments} \label{sec:experiments}
In this section, we compare the performance of \mkl~ and SGD for synthetic datasets for linear regression and small-scale neural networks.

\subsection{Linear Regression}
For simple linear regression, we assume that $X_i$ are sampled from normal distribution with different condition numbers. $X_i \sim \mathcal{N}(0, \bs{D})$ where $\bs{D}$ is a diagonal matrix such that $D_{11} = \kappa$ and $D_{ii}=1$ for all $i$). For the noisy case, we assume additive Gaussian noise with mean $0$ and variance $1$. We compare the performance of \mkl~and SGD for different values of $\kappa$ (Fig. \ref{fig:f2}) under noiseless and noisy settings against varying levels of corruption $\epsilon$. It is important to note that different $\kappa$ values correspond to different rates of convergence. To ensure fair comparison, we run the algorithms till the error values stop decaying and take the distance of $\ws$ from the exponential moving average of the iterates.

\begin{table*}[!h]
	\centering
	\begin{tabular}{|c|c|c|c|c|c|c|c|} 
		\hline
		\textbf{Dataset}  &\multicolumn{3}{c|}{\textbf{MNIST}} &\multicolumn{3}{c|}{\textbf{CIFAR10}}  \\ 	\hline
		\backslashbox{$\qquad \qquad~~~~\epsilon$}{Optimizer} &\textbf{SGD} &\textbf{\mkl}  &\textbf{Oracle} &\textbf{SGD} &\textbf{\mkl}  &\textbf{Oracle}\\ \hline 
		$0.1$ &\textbf{96.76}  &96.49 &98.52 &79.1  &\textbf{81.94} &84.56 \\  \hline
		$0.2$ &92.54 &\textbf{95.76} &98.33 &72.29 &\textbf{77.77} &84.40\\ \hline
		$0.3$ &85.77 &\textbf{95.96} &98.16 &63.96 &\textbf{66.49} &84.66\\ \hline
		$0.4$ &71.95 &\textbf{94.20} &97.98 &52.4 &\textbf{53.57} &84.42 \\ \hline
	\end{tabular}
	\begin{center}
		\caption{Comparing the test accuracy of SGD and \mkl ($k=5/3$) over MNIST and CIFAR-10 datasets in presence of corruptions via directed label noise.}	    
	\end{center}{}
\end{table*}

\subsection{Neural Networks}

For deep learning experiments, our results are in presence of corruptions via the directed noise model. In this corruption model, all the samples of class $a$ that are in error are assigned the same wrong label $b$. This is a stronger corruption model than corruption by random noise (results in Appendix). For the MKL-SGD algorithm, we run a more practical batched (size $b$) variant such that if $k=2$ the algorithm picks $b/2$ samples out of $b$ sample loss evaluations. The oracle contains results obtained by running SGD over only non-corrupted samples. More experimental results on neural networks for MNIST and CIFAR10 datasets can be found in the Appendix. 

\vspace{-0.2cm}
\paragraph{MNIST:} We train standard 2 layer convolutional network on subsampled MNIST ($5000$ samples with labels).  We train over 80 epochs using an initial learning rate of $0.05$ with the decaying schedule of factor $5$ after every $30$ epochs.  The results of the MNIST dataset are averaged over 5 runs. \vspace{-0.2cm}
\paragraph{CIFAR10:} We train Resnet-18 \citep{he2016identity} on CIFAR-10 ($50000$ training samples with labels) for over $200$ epochs using an initial learning rate of $0.05$ with the decaying schedule of factor $5$ after every $90$ epochs. The reported accuracy is based on the true validation set. The results of the CIFAR-10 dataset are averaged over $3$ runs.

\noindent
Lastly, in Fig. \ref{fig:my_label}, we show that for a neural network MKL-SGD typically has a higher training loss but smaller test loss which partially explains its superior generalization performance.

\begin{figure*}[!ht]
	\centering
	\includegraphics[scale=0.22]{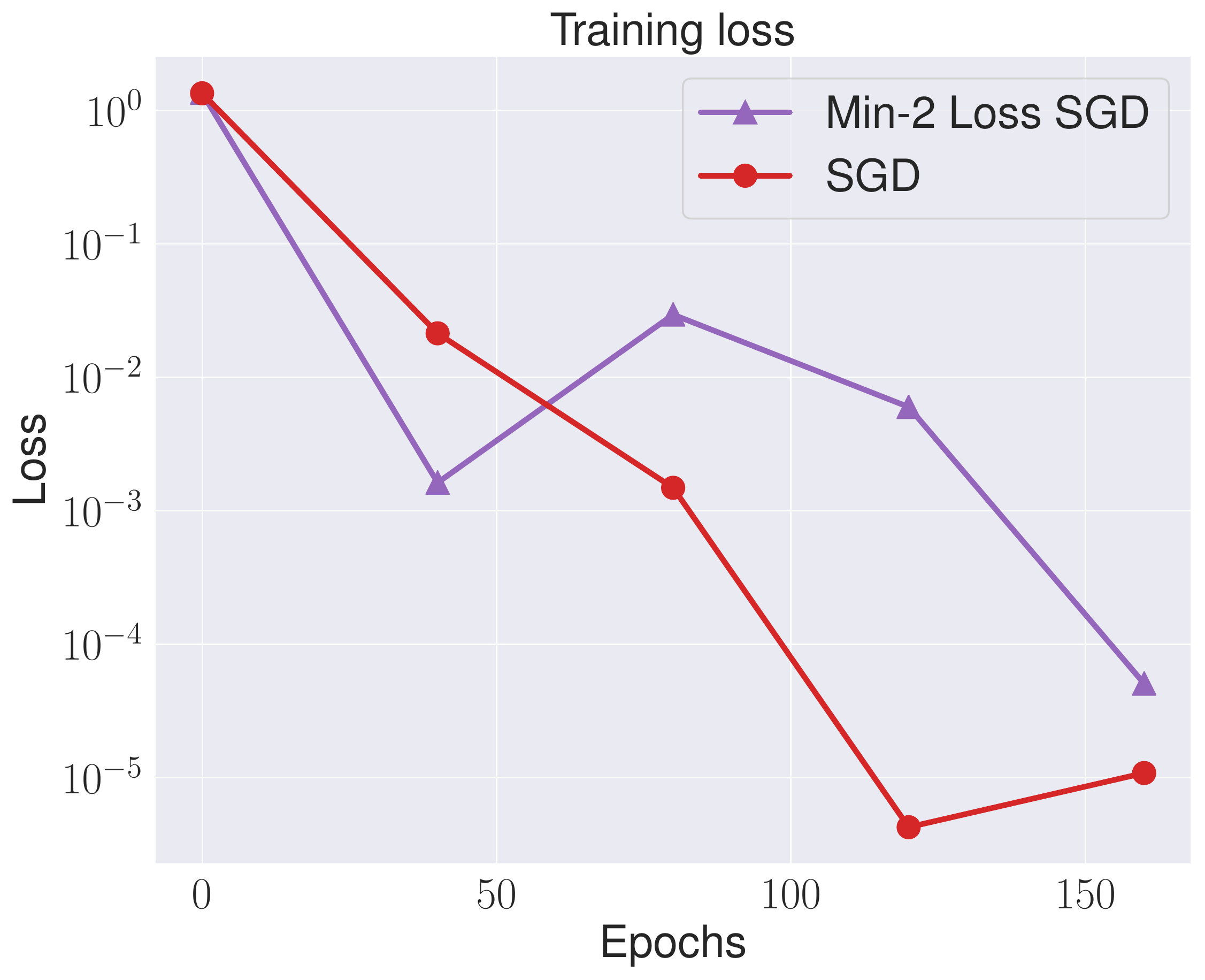}
	\includegraphics[scale=0.22]{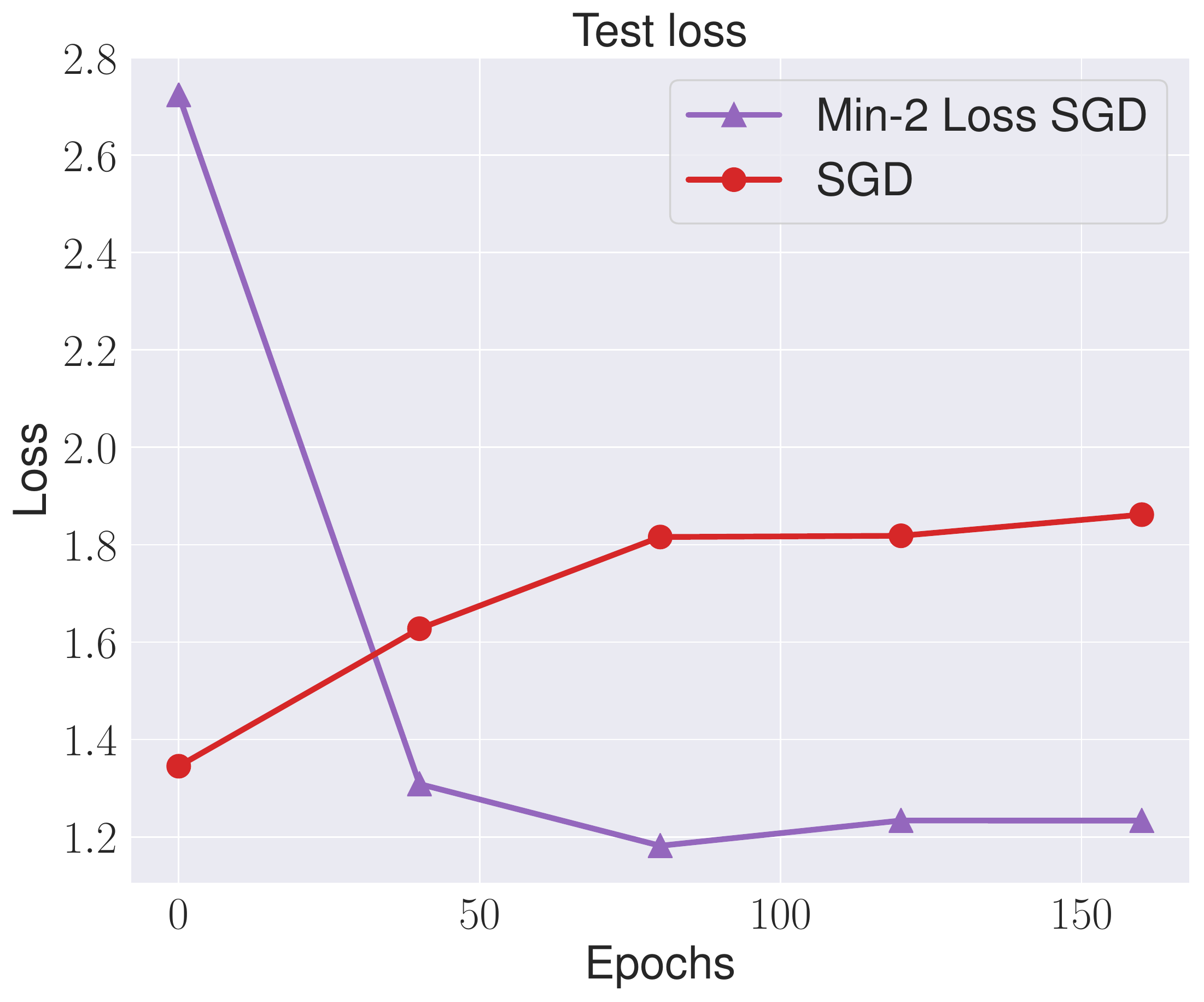}
	\includegraphics[scale=0.22]{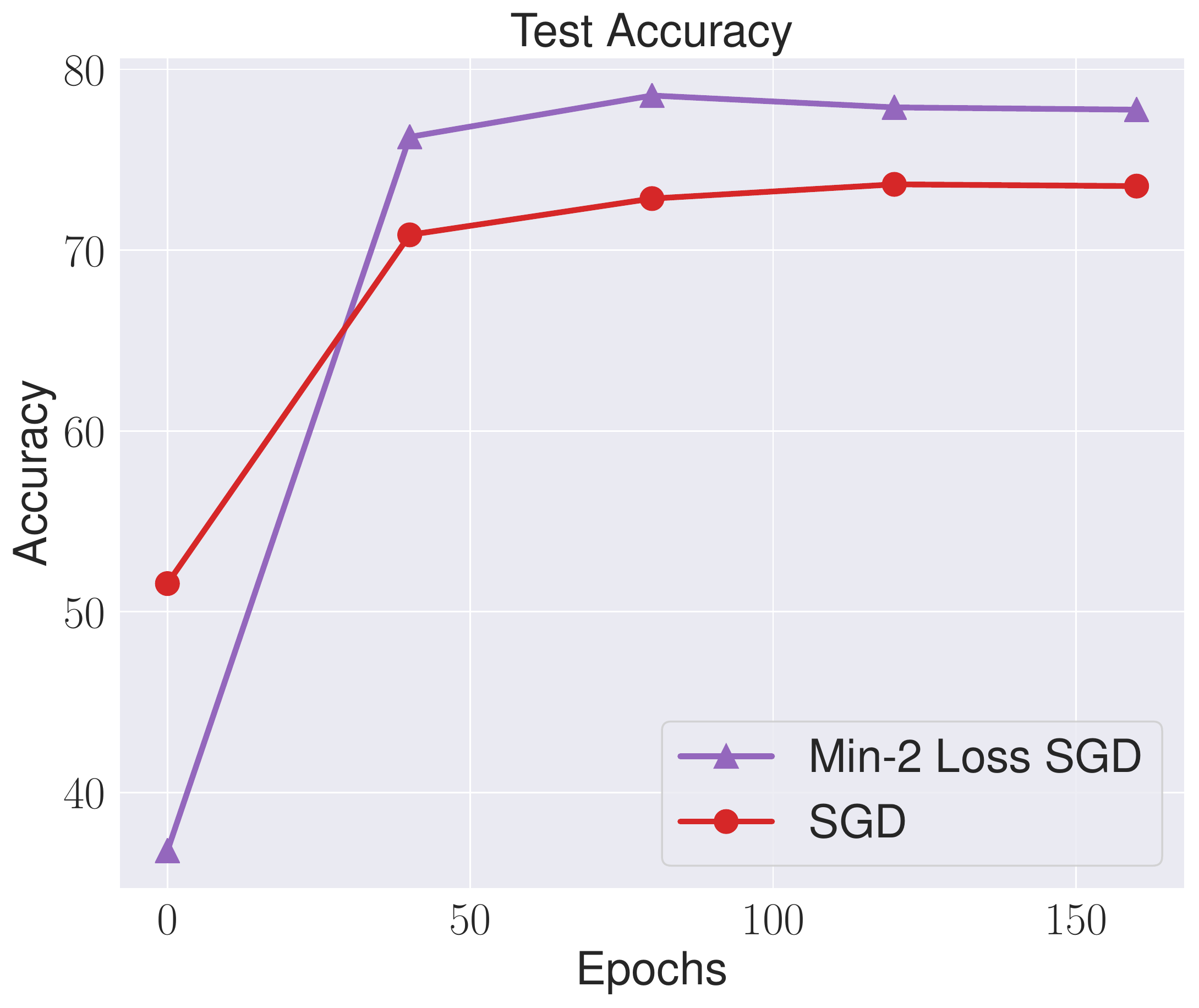}
	\caption{Comparing training loss, test loss and test accuracy of \mkl and SGD. Parameters: $\epsilon = 0.2$, $k = 2$, $b=16$. The training loss is lower for SGD which means that SGD overfits to the noisy data. The lower test loss and higher accuracy demonstrates the robustness \mkl provides for corrupted data.}
	\label{fig:my_label}
\end{figure*}

\section{Conclusion and Future Work} \label{sec:concfw} 
In this paper, we propose \mkl~that is computationally inexpensive, has linear convergence (upto a certain neighborhood) and is robust against outliers. We analyze \mkl~ algorithm under noiseless and noisy settings with and without outliers. \mkl~outperforms SGD in terms of generalization for both  linear regression and neural network experiments. More importantly, \mkl~ opens up a plethora of challenging questions with respect to understanding convex optimization in a non-convex landscape.

To ensure consistency, i.e. $\norm{\wsm -\ws} \rightarrow 0$, we require that $k \geq n\epsilon + 1$. In all other cases, there will be a non-zero contribution from the outliers which keeps the \mkl~solution from exactly converging to $\ws$. In this paper, we consider unknown $\epsilon$ and thus $k$ should be treated as a hyperparameter. However, if we knew the fraction of corruption, then with the right $k$ and smart initialization, it is possible to guarantee consistency. For neural network experiments in the Appendix, we show that tuning $k$ as a hyperparameter can lead to significant improvements in performance in presence of outliers. 

Preliminary experiments indicate that smarter initialization techniques can improve the performance of \mkl. The obvious question then is to provide worst case guarantees for a larger subset of problems using smarter initialization techniques. It will be interesting to analyze the tradeoff between rates of convergence to \mkl~ and its robustness to outliers. The worst case analysis in the noisy setting with and without outliers also remains an open problem. 
\newpage

\bibliography{references}

\begin{thebibliography}{38}
\providecommand{\natexlab}[1]{#1}
\providecommand{\url}[1]{\texttt{#1}}
\expandafter\ifx\csname urlstyle\endcsname\relax
  \providecommand{\doi}[1]{doi: #1}\else
  \providecommand{\doi}{doi: \begingroup \urlstyle{rm}\Url}\fi

\bibitem[Anaraki and Hughes(2014)]{anaraki2014memory}
Farhad~Pourkamali Anaraki and Shannon Hughes.
\newblock Memory and computation efficient pca via very sparse random
  projections.
\newblock In \emph{International Conference on Machine Learning}, pages
  1341--1349, 2014.

\bibitem[Angluin and Laird(1988)]{angluin1988learning}
Dana Angluin and Philip Laird.
\newblock Learning from noisy examples.
\newblock \emph{Machine Learning}, 2\penalty0 (4):\penalty0 343--370, 1988.

\bibitem[Balakrishnan et~al.(2017)Balakrishnan, Du, Li, and
  Singh]{balakrishnan2017computationally}
Sivaraman Balakrishnan, Simon~S Du, Jerry Li, and Aarti Singh.
\newblock Computationally efficient robust sparse estimation in high
  dimensions.
\newblock In \emph{Conference on Learning Theory}, pages 169--212, 2017.

\bibitem[Bengio et~al.(2009)Bengio, Louradour, Collobert, and
  Weston]{bengio2009curriculum}
Yoshua Bengio, J{\'e}r{\^o}me Louradour, Ronan Collobert, and Jason Weston.
\newblock Curriculum learning.
\newblock In \emph{Proceedings of the 26th annual international conference on
  machine learning}, pages 41--48. ACM, 2009.

\bibitem[Bhatia et~al.(2015)Bhatia, Jain, and Kar]{bhatia2015robust}
Kush Bhatia, Prateek Jain, and Purushottam Kar.
\newblock Robust regression via hard thresholding.
\newblock In \emph{Advances in Neural Information Processing Systems}, pages
  721--729, 2015.

\bibitem[Bhatia et~al.(2017)Bhatia, Jain, Kamalaruban, and
  Kar]{bhatia2017consistent}
Kush Bhatia, Prateek Jain, Parameswaran Kamalaruban, and Purushottam Kar.
\newblock Consistent robust regression.
\newblock In \emph{Advances in Neural Information Processing Systems}, pages
  2110--2119, 2017.

\bibitem[Charikar et~al.(2017)Charikar, Steinhardt, and
  Valiant]{charikar2017learning}
Moses Charikar, Jacob Steinhardt, and Gregory Valiant.
\newblock Learning from untrusted data.
\newblock In \emph{Proceedings of the 49th Annual ACM SIGACT Symposium on
  Theory of Computing}, pages 47--60. ACM, 2017.

\bibitem[Chen et~al.(2013)Chen, Caramanis, and Mannor]{chen2013robust}
Yudong Chen, Constantine Caramanis, and Shie Mannor.
\newblock Robust sparse regression under adversarial corruption.
\newblock In \emph{International Conference on Machine Learning}, pages
  774--782, 2013.

\bibitem[Diakonikolas et~al.(2018)Diakonikolas, Kamath, Kane, Li, Steinhardt,
  and Stewart]{diakonikolas2018sever}
Ilias Diakonikolas, Gautam Kamath, Daniel~M Kane, Jerry Li, Jacob Steinhardt,
  and Alistair Stewart.
\newblock Sever: A robust meta-algorithm for stochastic optimization.
\newblock \emph{arXiv preprint arXiv:1803.02815}, 2018.

\bibitem[Diakonikolas et~al.(2019)Diakonikolas, Kamath, Kane, Li, Moitra, and
  Stewart]{diakonikolas2019robust}
Ilias Diakonikolas, Gautam Kamath, Daniel Kane, Jerry Li, Ankur Moitra, and
  Alistair Stewart.
\newblock Robust estimators in high-dimensions without the computational
  intractability.
\newblock \emph{SIAM Journal on Computing}, 48\penalty0 (2):\penalty0 742--864,
  2019.

\bibitem[Freund et~al.(1999)Freund, Schapire, and Abe]{freund1999short}
Yoav Freund, Robert Schapire, and Naoki Abe.
\newblock A short introduction to boosting.
\newblock \emph{Journal-Japanese Society For Artificial Intelligence},
  14\penalty0 (771-780):\penalty0 1612, 1999.

\bibitem[Gunasekar et~al.(2018)Gunasekar, Lee, Soudry, and
  Srebro]{gunasekar2018characterizing}
Suriya Gunasekar, Jason Lee, Daniel Soudry, and Nathan Srebro.
\newblock Characterizing implicit bias in terms of optimization geometry.
\newblock \emph{arXiv preprint arXiv:1802.08246}, 2018.

\bibitem[He et~al.(2016)He, Zhang, Ren, and Sun]{he2016identity}
Kaiming He, Xiangyu Zhang, Shaoqing Ren, and Jian Sun.
\newblock Identity mappings in deep residual networks.
\newblock In \emph{European conference on computer vision}, pages 630--645.
  Springer, 2016.

\bibitem[Huber(2011)]{huber2011robust}
Peter~J Huber.
\newblock \emph{Robust statistics}.
\newblock Springer, 2011.

\bibitem[Jiang et~al.(2017)Jiang, Zhou, Leung, Li, and
  Fei-Fei]{jiang2017mentornet}
Lu~Jiang, Zhengyuan Zhou, Thomas Leung, Li-Jia Li, and Li~Fei-Fei.
\newblock Mentornet: Learning data-driven curriculum for very deep neural
  networks on corrupted labels.
\newblock \emph{arXiv preprint arXiv:1712.05055}, 2017.

\bibitem[Kahn and Marshall(1953)]{kahn1953methods}
Herman Kahn and Andy~W Marshall.
\newblock Methods of reducing sample size in monte carlo computations.
\newblock \emph{Journal of the Operations Research Society of America},
  1\penalty0 (5):\penalty0 263--278, 1953.

\bibitem[Karimi et~al.(2016)Karimi, Nutini, and Schmidt]{karimi2016linear}
Hamed Karimi, Julie Nutini, and Mark Schmidt.
\newblock Linear convergence of gradient and proximal-gradient methods under
  the polyak-{\l}ojasiewicz condition.
\newblock In \emph{Joint European Conference on Machine Learning and Knowledge
  Discovery in Databases}, pages 795--811. Springer, 2016.

\bibitem[Karmalkar et~al.(2019)Karmalkar, Klivans, and
  Kothari]{karmalkar2019list}
Sushrut Karmalkar, Adam Klivans, and Pravesh Kothari.
\newblock List-decodable linear regression.
\newblock In \emph{Advances in Neural Information Processing Systems}, pages
  7423--7432, 2019.

\bibitem[Katharopoulos and Fleuret(2018)]{katharopoulos2018not}
Angelos Katharopoulos and Fran{\c{c}}ois Fleuret.
\newblock Not all samples are created equal: Deep learning with importance
  sampling.
\newblock \emph{arXiv preprint arXiv:1803.00942}, 2018.

\bibitem[Klivans et~al.(2018)Klivans, Kothari, and Meka]{klivans2018efficient}
Adam Klivans, Pravesh~K Kothari, and Raghu Meka.
\newblock Efficient algorithms for outlier-robust regression.
\newblock \emph{arXiv preprint arXiv:1803.03241}, 2018.

\bibitem[Kumar et~al.(2010)Kumar, Packer, and Koller]{kumar2010self}
M~Pawan Kumar, Benjamin Packer, and Daphne Koller.
\newblock Self-paced learning for latent variable models.
\newblock In \emph{Advances in Neural Information Processing Systems}, pages
  1189--1197, 2010.

\bibitem[Lai et~al.(2016)Lai, Rao, and Vempala]{lai2016agnostic}
Kevin~A Lai, Anup~B Rao, and Santosh Vempala.
\newblock Agnostic estimation of mean and covariance.
\newblock In \emph{2016 IEEE 57th Annual Symposium on Foundations of Computer
  Science (FOCS)}, pages 665--674. IEEE, 2016.

\bibitem[Lee and Sidford(2013)]{lee2013efficient}
Yin~Tat Lee and Aaron Sidford.
\newblock Efficient accelerated coordinate descent methods and faster
  algorithms for solving linear systems.
\newblock In \emph{2013 IEEE 54th Annual Symposium on Foundations of Computer
  Science}, pages 147--156. IEEE, 2013.

\bibitem[Liu et~al.(2018)Liu, Shen, Li, and Caramanis]{liu2018high}
Liu Liu, Yanyao Shen, Tianyang Li, and Constantine Caramanis.
\newblock High dimensional robust sparse regression.
\newblock \emph{arXiv preprint arXiv:1805.11643}, 2018.

\bibitem[Moulines and Bach(2011)]{moulines2011non}
Eric Moulines and Francis~R Bach.
\newblock Non-asymptotic analysis of stochastic approximation algorithms for
  machine learning.
\newblock In \emph{Advances in Neural Information Processing Systems}, pages
  451--459, 2011.

\bibitem[Needell et~al.(2014)Needell, Ward, and Srebro]{needell2014stochastic}
Deanna Needell, Rachel Ward, and Nati Srebro.
\newblock Stochastic gradient descent, weighted sampling, and the randomized
  kaczmarz algorithm.
\newblock In \emph{Advances in neural information processing systems}, pages
  1017--1025, 2014.

\bibitem[Owen(2007)]{owen2007robust}
Art~B Owen.
\newblock A robust hybrid of lasso and ridge regression.
\newblock 2007.

\bibitem[Prasad et~al.(2018)Prasad, Suggala, Balakrishnan, and
  Ravikumar]{prasad2018robust}
Adarsh Prasad, Arun~Sai Suggala, Sivaraman Balakrishnan, and Pradeep Ravikumar.
\newblock Robust estimation via robust gradient estimation.
\newblock \emph{arXiv preprint arXiv:1802.06485}, 2018.

\bibitem[Ren et~al.(2018)Ren, Zeng, Yang, and Urtasun]{ren2018learning}
Mengye Ren, Wenyuan Zeng, Bin Yang, and Raquel Urtasun.
\newblock Learning to reweight examples for robust deep learning.
\newblock \emph{arXiv preprint arXiv:1803.09050}, 2018.

\bibitem[Rolnick et~al.(2017)Rolnick, Veit, Belongie, and
  Shavit]{rolnick2017deep}
David Rolnick, Andreas Veit, Serge Belongie, and Nir Shavit.
\newblock Deep learning is robust to massive label noise.
\newblock \emph{arXiv preprint arXiv:1705.10694}, 2017.

\bibitem[Rousseeuw(1984)]{rousseeuw1984least}
Peter~J Rousseeuw.
\newblock Least median of squares regression.
\newblock \emph{Journal of the American statistical association}, 79\penalty0
  (388):\penalty0 871--880, 1984.

\bibitem[Shen and Sanghavi(2019)]{shen2019learning}
Yanyao Shen and Sujay Sanghavi.
\newblock Learning with bad training data via iterative trimmed loss
  minimization.
\newblock In \emph{International Conference on Machine Learning}, pages
  5739--5748, 2019.

\bibitem[Strohmer and Vershynin(2009)]{strohmer2009randomized}
Thomas Strohmer and Roman Vershynin.
\newblock A randomized kaczmarz algorithm with exponential convergence.
\newblock \emph{Journal of Fourier Analysis and Applications}, 15\penalty0
  (2):\penalty0 262, 2009.

\bibitem[V{\'\i}{\v{s}}ek et~al.(2002)]{vivsek2002least}
Jan V{\'\i}{\v{s}}ek et~al.
\newblock The least weighted squares ii. consistency and asymptotic normality.
\newblock \emph{Bulletin of the Czech Econometric Society}, 9, 2002.

\bibitem[V{\'\i}{\v{s}}ek(2006)]{vivsek2006least}
Jan~{\'A}mos V{\'\i}{\v{s}}ek.
\newblock The least trimmed squares. part i: Consistency.
\newblock \emph{Kybernetika}, 42\penalty0 (1):\penalty0 1--36, 2006.

\bibitem[Xu et~al.(2009)Xu, Caramanis, and Mannor]{xu2009robust}
Huan Xu, Constantine Caramanis, and Shie Mannor.
\newblock Robust regression and lasso.
\newblock In \emph{Advances in Neural Information Processing Systems}, pages
  1801--1808, 2009.

\bibitem[Zhang(2017)]{zhang2017restricted}
Hui Zhang.
\newblock The restricted strong convexity revisited: analysis of equivalence to
  error bound and quadratic growth.
\newblock \emph{Optimization Letters}, 11\penalty0 (4):\penalty0 817--833,
  2017.

\bibitem[Zhao and Zhang(2015)]{zhao2015stochastic}
Peilin Zhao and Tong Zhang.
\newblock Stochastic optimization with importance sampling for regularized loss
  minimization.
\newblock In \emph{international conference on machine learning}, pages 1--9,
  2015.

\end{thebibliography}

\newpage
\newpage
\onecolumn

\section{Appendix}
\subsection{Additional Results for Section 3}

The following lemma provides upper bounds on the expected gradient of the worst-possible \mkl~solution that lies in a ball around $\ws$. Simultaneously satisfying the following bound with the one in Lemma \ref{lemma:sgdbound} may lead to an infeasible set of $\epsilon$ and $N'$. And thus we use Lemma \ref{lemma:m2lboundstrongconvexity} in conjunction with \ref{lemma:sgdbound}.

\begin{lemma} \label{lemma:m2lbound}
	Let us assume that \mkl~converges to $\wsm$. For any $\wsm \in  \mathcal{B}_r(\ws)$ that satisfies assumptions N1, N2, A4 and A5, there exists $N'\geq N$ and $\epsilon' \leq \epsilon$ such that, 
	\begin{align*}
	\norm{\sum_{i\notin\mathbb{O}} p_i(\wsm) \nabla f_i(\wsm)} \leq \min\left\{(1 - \epsilon^k) L \norm{\wsm - \ws}, \epsilon^k G(\w)\right\}
	\end{align*}
\end{lemma}

The proof for lemma 2 can be found in the Appendix Section \ref{sec:lemmam2lbound}

\subsection{Proofs and supporting lemmas}

\subsubsection{Proof of Lemma \ref{lemma:sc}}
\begin{proof}
	$\widetilde{F}(\bs{w}) = \sum_i p_{m_i(\bs{w})}(\bs{w})$. 
	Let us fix a $\bs{w}$ such that $p_i = p_i(\bs{w})$.  We know that for any $p_i$, $\sum_i p_i f_i(\bs{w})$ is strongly convex in $\bs{w}$ with parameter $\lambda_{\bs{w}}$. This implies
	$$\nabla \widetilde{F}(\bs{w})^\top (\bs{w} - \ws)  \geq \lambda_{\bs{w} } \norm{\bs{w} - \ws}^2$$
\end{proof}

A naive bound for the above Lemma can be: 
$$\nabla \widetilde{F}(\bs{w})^\top (\bs{w} - \ws)  \geq \min_{i}p_{i}\sum_i  f_i(\bs{w})\geq \underbrace{\lambda \min_{i}p_{i}}_{\lambda_{\bs{w} }}\norm{\bs{w} - \ws}^2$$

\subsubsection{Proof of Theorem \ref{theorem:nolocalminima}} 
\label{sec:proofnolocalminima}

\begin{proof}
	By the definition of the noiseless framework, $\ws$ is the unique optimum of $F(\bs{w})$ and lies in the optimal set of each $f_i(.)$. We will prove this theorem by contradiction. Assume there exists some $\hat{\w} \neq \ws$ that also satisfies optimum of $\nabla \widetilde{F}(\hat{\w}) = \bs{0}$. 
	At $\hat{\w}$, we have $0 =<\nabla \widetilde{F}(\hat{\w}), \hat{\w} - \ws> = \lambda \norm{\hat{\w} - \ws}^2 $. This implies $\hat{\w} = \ws$. 
\end{proof}

Theorem 1 and Assumption 2 guarantee that $\lambda_{\w}> 0$. If $f(\w)$ is strongly convex and $g(\w)$ is convex, then we know that $f(\w) + g(\w)$ is strongly convex. On similar lines we can show that $\lambda > 0$ by splitting the terms in $\tilde{F(\w)}$ as $p_{min} F(\w)$ and $(\tilde{F}(\w) - p_{min} F(\w))$. The first term has $\lambda > 0$ (Assumption 2) and the second term has $\lambda = 0$ (since it is convex). Note, $p_{min}$ is a positive constant independent of $\w$ and so the above lemma is for all $\w$.

\subsubsection{Proof of Lemma \ref{lemma:vector}} \label{sec:vector}

Let $\bar{\w}$ be a stationary point of \mkl.
Now, we analyze the loss landscape on the line joining $\ws$ and $\w_C$ where $\w_C = C \bar{\w}$ is any arbitrary point \footnote{Note that we just need $\w_C$ for the purpose of landscape analysis and it is not a parameter of the algorithm} in the landscape at a distance as far as the farthest outlier from $\ws$. Let $C$ be a very large number.

The loss functions and $\widetilde{\w}$ are redefined as follows: 
\begin{align*} 
&f_i(\w) = \begin{cases}{}
&l_i \norm{\w -\ws}^2 ~~\forall~i \in \mathbb{O} \notag \\
&l_i \norm{\w -\w_{b_i}}^2 ~~ \forall~i \notin \mathbb{O},  \notag
\end{cases}\\
&\tw :=\begin{cases} \begin{drcases} \w \bigg{|} &\w = \min_{\alpha\in (0, 1)} \alpha \ws + (1-\alpha)\w_{C} , \\
& f_{l_{m}}(\w) = f_{l_{M}}(\w) \end{drcases}\end{cases}
\end{align*}

where $\abs{\mathbb{O}} = b$ such that $n=g+b$.  Let $l_{m} = \min_{i\notin\mathbb{O}} l_i$ and Let $l_{M} = \max_{i\in\mathbb{O}} l_i$ and $l_{max} = \min_{i\in[n]} l_i$, $l_{min} = \min_{i \in [n]}l_i$. Let us define $\kappa=\dfrac{l_{max}}{l_{min}} \geq \dfrac{l_{M}}{l_{m}}$.

Now at $\bar{w}$, we have $\nabla \widetilde{F}(\bar{\w}) = 0$. Let us assume that the outliers are chosen in such a way that at $\w_C$, all the outliers have the lowest loss. As stated in the previous lemma, the results hold irrespective of that. 
This implies:
\begin{align*}
\sum_{i\notin \mathbb{O}} p_i(\w_C) \nabla f_i(\bar{\w}) &= - \sum_{j \in \mathbb{O}}  p_j(\w_C) \nabla f_j(\bar{\w}) \\
\sum_{i\notin \mathbb{O}} p_i(\w_C) l_i (\bar{\w} - \ws)  &= - \sum_{j \in \mathbb{O}}  p_j(\w_C) l_j (\bar{\w} - \w_{b_j}) \\
\bar{\w} &= \dfrac{ \sum_{i\notin \mathbb{O}} p_i(\w_C) l_i \ws + \sum_{j \in \mathbb{O}}  p_j(\w_C) l_j \w_{b_j}}{ \sum_{i\notin \mathbb{O}} p_i(\w_C) l_i + \sum_{j \in \mathbb{O}} p_j(\w_C) l_j } \\
\text{By triangle inequality, }
\norm{\bar{\w} - \ws} &\leq \dfrac{\sum_{j \in \mathbb{O}}  p_j(\w_C) l_j \norm{\w_{b_j}-\ws}}{ \sum_{i\notin \mathbb{O}} p_i(\w_C) l_i + \sum_{j \in \mathbb{O}} p_j(\w_C) l_j }
\end{align*}{}

Without loss of generality assume that the outliers are ordered as follows:
$\norm{\w_{b_1}-\ws} \leq \norm{\w_{b_2}-\ws}\leq \dots \leq \norm{\w_{b_{|\mathbb{O}|}}-\ws}  $.

Now $\widetilde{\w}$ be some point of intersection of function in the set of clean samples and a function in the set of outliers to $\ws$. Let $\theta_j$ be the angle between the line connecting $\w_{b_j}$ and $\ws$ to the line connecting $\w_C$ to $\ws$. For any two curves with Lipschitz constants $l_i$ and $l_j$, the halfspaces passing through the weighted mean are also the region where both functions have equal values.

Thus, $$\widetilde{\w} = \dfrac{\sqrt{l_i} \ws +\sqrt{l_j} \w_{b_j}}{\sqrt{l_i} + \sqrt{l_j} }$$.

\begin{align*}
\norm{\widetilde{\w} -\ws} &= \dfrac{\sqrt{l_j} \norm{\w_{b_j}-\ws}}{\sqrt{l_j} + \sqrt{l_i}}
\end{align*}{}

Let $\gamma$ denote the following ratio: $$\gamma = \dfrac{\min_{j\in\mathbb{O}}\norm{\w_{b_j}-\ws}}{\max_{j\in\mathbb{O}}\norm{\w_{b_j}-\ws}}=\dfrac{2\delta}{\delta_{max}}$$
Now, we want:
\begin{align*}{}
\dfrac{\sum_{j \in \mathbb{O}}  p_j(\w_C) l_j \norm{\w_{b_j}-\ws}}{ \sum_{i\notin \mathbb{O}} p_i(\w_C) l_i + \sum_{j \in \mathbb{O}} p_j(\w_C) l_j } &\leq  \dfrac{\sqrt{l_{t_j}} }{\sqrt{l_{t_j}} + \sqrt{l_{g}}}\dfrac{\norm{\w_{b_j}-\ws}}{\cos{\theta_j}} = \dfrac{\norm{\widetilde{\w} -\ws}}{\cos{\theta_j}} \\
\dfrac{\sum_{j \in \mathbb{O}}  p_j(\w_C) l_j \norm{\w_{b_j}-\ws}}{ \sum_{i\notin \mathbb{O}} p_i(\w_C) l_i + \sum_{j \in \mathbb{O}} p_j(\w_C) l_j } \leq \dfrac{\sum_{j \in \mathbb{O}}  p_j(\w_C) l_j \norm{\w_{b_{|\mathbb{O}|}}-\ws}}{ \sum_{i\notin \mathbb{O}} p_i(\w_C) l_i + \sum_{j \in \mathbb{O}} p_j(\w_C) l_j } &\leq \dfrac{\sqrt{l_{t_j}} }{\sqrt{l_{t_j}} + \sqrt{l_{g}}}\dfrac{\norm{\w_{b_j}-\ws}}{\cos{\theta_j}} \\
\dfrac{\sum_{j \in \mathbb{O}}  p_j(\w_C) l_j }{ \sum_{i\notin \mathbb{O}} p_i(\w_C) l_i + \sum_{j \in \mathbb{O}} p_j(\w_C) l_j } &\leq \dfrac{\sqrt{l_{t_j}} }{\sqrt{l_{t_j}} + \sqrt{l_{g}}}\dfrac{\norm{\w_{b_j}-\ws}}{\cos{\theta_j}\norm{\w_{b_{|\mathbb{O}|}}-\ws}} \\
\dfrac{\sum_{j \in \mathbb{O}}  p_j(\w_C) l_j }{ \sum_{i\notin \mathbb{O}} p_i(\w_C) l_i + \sum_{j \in \mathbb{O}} p_j(\w_C) l_j } &\leq \dfrac{\sqrt{l_{t_j}} }{\sqrt{l_{t_j}} + \sqrt{l_{g}}}\dfrac{\gamma}{\cos{\theta_j}}
\end{align*}

For simplicity, $\Gamma = \dfrac{\gamma}{\cos{\theta_j}} $,
then we have:
\begin{align*}{}
\dfrac{\sum_{j \in \mathbb{O}}  p_j(\w_C) l_j }{ \sum_{i\notin \mathbb{O}} p_i(\w_C) l_i + \sum_{j \in \mathbb{O}} p_j(\w_C) l_j } &\leq \dfrac{\sqrt{l_{t_j}} }{\sqrt{l_{t_j}} + \sqrt{l_{g}}}\Gamma \\
\dfrac{1}{\Gamma} \left(\dfrac{\sqrt{l_{g}}}{\sqrt{l_{t_j}}} + 1\right) - 1 &\leq \dfrac{ (1-\hat{p}) l_m  }{\hat{p} l_M} \leq \dfrac{ \sum_{i\notin \mathbb{O}} p_i(\w_C) l_i  }{\sum_{j \in \mathbb{O}}  p_j(\w_C) l_j } \\
\dfrac{\hat{p}}{1-\hat{p}} &\leq \dfrac{\frac{l_m}{l_M}}{\frac{1}{\Gamma} - 1 +\frac{1}{\Gamma}\frac{\sqrt{l_{g}}}{\sqrt{l_{t_j}}}}\\
\hat{p} &\leq \dfrac{1}{1 + \kappa \left(\frac{1}{\Gamma} - 1 +\frac{\sqrt{\kappa}}{\Gamma}\right)} \leq \dfrac{1}{ 1 + \frac{l_M}{l_m}\left(\frac{1}{\Gamma} - 1 +\frac{1}{\Gamma}\frac{\sqrt{l_{g}}}{\sqrt{l_{t_j}}}\right)}
\end{align*}

Replacing  $\Gamma = \dfrac{\gamma}{\min_j\cos{\theta_j}} $,
and let $q = \frac{\cos{\theta_j}}{\gamma} - 1 +\frac{\cos{\theta_j}\sqrt{\kappa}}{\gamma} $
the condition to guarantee that bad local minima do no exist is 
$   \hat{p} \leq \dfrac{1}{1 + \kappa q} $
and $q > 0$.
Now, we can repeat the above analysis recursively for every corresponding $\w_C$ and $\tilde{\w}$ in the landscape and so now $\theta_j$ is a function of $\tilde{\w}$ as well as is represented in the theorem statement.

\textbf{Note:} In the vector case, for example there exists a fine tradeoff between how large $\theta_j$ can be and if for large $\theta_j$, the loss corresponding to the outlier will be one of the lowest. Understanding that tradeoff is beyond the scope of this paper.

\subsubsection{Proof of Lemma \ref{lemma:sgdbound}} \label{sec:lemmasgdbound}

\begin{proof}
	At $\wsgd$, $\nabla \widetilde{F}(\wsgd) = 0$. Then, 
	\begin{align}
	\sum_{i\notin\mathbb{O}} \nabla f_i(\wsgd) &= -  \sum_{i\in\mathcal{O}} \nabla f_i(\wsgd) \notag \\
	\norm{\sum_{i\notin\mathbb{O}} \nabla f_i(\wsgd) }&= \norm{ \sum_{i\in\mathcal{O}} \nabla f_i(\wsgd)}\notag \\
	\norm{\sum_{i\notin\mathbb{O}}   \nabla f_i(\wsgd)} &\leq \sum_i  \norm{\nabla f_i(\wsgd)} \notag \\
	&\leq \sum_i  L \norm{\wsgd - \ws}\notag\\
	&=(1 - \epsilon) n L \norm{\wsgd - \ws}\\
	\norm{ \sum_{i\in\mathcal{O}} \nabla f_i(\wsgd)} &\leq\sum_{i\in\mathcal{O}}  \norm{  \nabla f_i(\wsgd)} \notag \\
	&\leq \sum_{i\in\mathcal{O}} G(\wsgd) \notag\\
	&\leq \epsilon  G(\wsgd)
	\end{align}
	
	$$\norm{ \sum_{i\in\mathcal{O}} \nabla f_i(\wsgd)} = \min\left(\epsilon n  G(\wsgd), (1 - \epsilon) n L \norm{\wsgd - \ws}\right)$$
\end{proof}

\subsubsection{Proof of Lemma \ref{lemma:m2lboundstrongconvexity}} \label{sec:lemmam2lboundstrongconvexity}

\begin{proof}
	At $\wsm$, $\nabla \widetilde{F}(\wsm) = 0$. This implies
	\begin{align}
	\sum_{i\notin\mathbb{O}} p_i(\wsm)\nabla f_i(\wsm) &= -  \sum_{i\in\mathcal{O}}p_i(\wsm) \nabla f_i(\wsm) \notag \\
	\text{Multiplying both sides by } (\wsm - \ws) \qquad&\notag\\
	\sum_{i\notin\mathbb{O}} p_i(\wsm)<\nabla f_i(\wsm) , \wsm - \ws> &= -	\sum_{i\in\mathcal{O}} p_i(\wsm)<\nabla f_i(\wsm) , \wsm - \ws> \\
	<\nabla\widetilde{F}_\mathcal{G}(\wsm) , \wsm - \ws>  &= - \sum_{i\in\mathcal{O}} p_i(\wsm)<\nabla f_i(\wsm) , \wsm - \ws> \notag\\
	\text{Lower bounding the LHS using Lemma~ \ref{lemma:sc} and }&m= m(\wsm)\footnote{Note that strong convexity holds since we are only considering the set of good samples}&,\\
	m \norm{ \wsm - \ws}^2&\leq \norm{	<\nabla\widetilde{F}_\mathcal{G}(\wsm) , \wsm - \ws>} =LHS \\
	RHS &\leq \norm{ - \sum_{i\in\mathcal{O}} p_i(\wsm)<\nabla f_i(\wsm) , \wsm - \ws>} \notag\\
	m \norm{ \wsm - \ws}^2 &\leq \sum_{i\in\mathcal{O}}p_i(\wsm)\norm{  <\nabla f_i(\wsm) , \wsm - \ws>}\notag \\
	m \norm{ \wsm - \ws}^2 &\leq \sum_{i\in\mathcal{O}}p_i(\wsm)\norm{ \nabla f_i(\wsm)} \norm{ \wsm - \ws} \notag \\
	m \norm{ \wsm - \ws}^2 &\leq \sum_{i\in\mathcal{O}}p_i(\wsm) \norm{ \wsm - \ws} G(\wsgd) \notag \\
	m \norm{ \wsm - \ws} &\leq \epsilon^k G(\wsgd) 
	\end{align}
\end{proof}

\subsubsection{Proof of Theorem \ref{theorem:relativebounds}} \label{sec:theoremrelativebounds}

\begin{proof}
	There exists an $\epsilon'\leq\epsilon$ such that in Lemma  \ref{lemma:sgdbound}, we have 
	\begin{align*}
	(1 - \epsilon)  L \norm{\wsgd - \ws}\geq \epsilon G(\wsgd)
	\end{align*}
	Combining above equation with Lemma \ref{lemma:m2lboundstrongconvexity}, we get
	\begin{align*}
	(1 - \epsilon)  L \norm{\wsgd - \ws}&\geq \epsilon G(\wsgd)
	\geq \epsilon \frac{\lambda}{\epsilon^2} \norm{ \wsm - \ws} \\
	\Rightarrow \norm{ \wsm - \ws} &\leq \dfrac{(1 - \epsilon)  L \epsilon^{k-1} }{\lambda} \norm{\wsgd - \ws}
	\end{align*}
	Picking a large enough $k$, we can guarantee that $\dfrac{(1 - \epsilon)  L \epsilon^{k-1} }{\lambda} < 1$
\end{proof}


\subsubsection{Proof of Lemma \ref{lemma:m2lbound}} \label{sec:lemmam2lbound}

\begin{proof}
	From the definition of good samples in the noiseless setting, we know that $f_i(\ws) = 0 ~\forall ~i\notin \mathbb{O}$. Similarly, for samples belonging to the outlier set, $f_i(\ws) > 0 ~\forall ~i \in \mathcal{O}$.  There exists a ball around the optimum of radius $r$ such that $f_i(\bs{w}) \leq f_j(\bs{w}) ~\forall i \notin \mathbb{O}, j\in \mathcal{O}, \bs{w} \in \mathbb{O}_r(\ws)$. Assume that $N'\geq N$ and $\epsilon' \leq \epsilon$, such that $\norm{ \wsm-\ws}\leq r$.
	
	At $\wsm$, $\nabla \widetilde{F}(\wsm) = 0$. This implies
	\begin{align}
	\sum_{i\notin\mathbb{O}} p_i(\wsm)\nabla f_i(\wsm) &= -  \sum_{i\in\mathcal{O}}p_i(\wsm) \nabla f_i(\wsm) \notag \\
	\norm{\sum_{i\notin\mathbb{O}} p_i(\wsm) \nabla f_i(\wsm)} &= \norm{ \sum_{i\in\mathcal{O}} p_i(\wsm) \nabla f_i(\wsm)} \notag \\
	\norm{\sum_{i\notin\mathbb{O}} p_i(\wsm)  \nabla f_i(\wsm)} &\leq \sum_i p_i(\wsm) \norm{\nabla f_i(\wsm)} \notag \\
	&\leq \sum_i p_i(\wsm) L \norm{\wsm - \ws}\notag\\
	&=(1 - \epsilon^k)  L \norm{\wsm - \ws}\\
	\norm{ \sum_{i\in\mathcal{O}} p_i(\wsm) \nabla f_i(\wsm)} &\leq\sum_{i\in\mathcal{O}} p_i(\wsm) \norm{  \nabla f_i(\wsm)} \notag \\
	&\leq \sum_{i\in\mathcal{O}} p_i(\wsm) G(\wsm) \notag\\
	&\leq \epsilon^k G(\wsm)
	\end{align}
	
\end{proof}

\subsection{Additional results and proofs for Section \ref{sec:convergence}}

Consider the sample size $n$ with bad set(outlier) $\mathbb{O}$ and  good set $\mathcal{G}$ such that $|\mathcal{G}|=n-|\mathbb{O}|$. Define
$${F}_{good}(\bs{w})=\frac{1}{\abs{\mathcal{G}}}\sum_{i\in {\mathcal{G}}}f_i(\bs{w}) .$$
We assume:\\
$(1)$ (Stationary Point)  
Assume $\wg$ is the solution for the average loss function of good sample such that 
$$\nabla{F}_{good}(\wg)= 0  \quad \text{ but }  \nabla f_i(\wg)\neq 0, 
\forall i\in \mathbb{O}$$
$(2)$  (Strong Convexity)  $ {F}_{good}(\bs{w})$ is strongly convex with parameters $\lambda_{good}$ i.e.,
$$\langle{\nabla F_{good}(\bs{w}) -\nabla F_{good}(\wg), \bs{w} - \wg\rangle}~ \geq \lambda_{good} \norm{ \bs{w} - \wg}^2$$
$(3)$  (Gradient Lipschitz)  $ {f}_{i}(\bs{w})$ has $L_i$ Liptchitz gradient i.e.,
$$\|\nabla f_{i}(\bs{w}) -\nabla f_{i}(\wg)\|~ \leq L_{i} \norm{ \bs{w} - \wg}$$

\begin{theorem} (\textbf{Distance to $\wg$})  \label{lem:R}
	\begin{align}
	\mathbb{E}_i\left[ \norm{\bs{w_{t+1}} - \wg}^2|\bs{w_{t}} \right]
	\leq & \left( 1-  2\eta_t  \lambda_{good }(1-\eta_t \sup_iL_i)\min_{ i\in \mathcal{G}}p_i(\bs{w_t}) \right)\norm{\bs{w_t} -\wg}^2+ R_t 
	\end{align}
	where 
	\begin{align*}
	R_t 
	=	& - 2  \eta_t \sum_{i\in \mathcal{G}} p_i(\bs{w_t}) \langle{ \bs{w_t} - \wg,\nabla f_{i}( \wg)\rangle}\\
	& + 2\eta_t^2\sum_{i\in \mathcal{G}} p_i(\bs{w_t}) \norm{\nabla f_{i}( \wg)}^2 + \eta_t^2\sum_{i\in \mathbb{O}}p_i(\bs{w_t})\norm{\nabla f_{i}( \bs{w_{t}})}^2+2  \eta_t \sum_{i\in \mathbb{O}}p_i(\bs{w_t}) \left(f_{i}(\wg)- f_i(\bs{w_t} )\right)
	\end{align*}	
\end{theorem}

\begin{proof}
	Observe first that for each component function i.e. ,
	$$ \langle{\bs{w} -\bs{v}, \nabla f_{i}(\bs{w})-\nabla f_{i}(\bs{v})\rangle} \geq 
	\frac{1}{L_i}\| f_i(\bs{w}) -f_i(\bs{v}) \|^2$$
	For detailed proof, see Lemma A.1 in \citep{needell2014stochastic}.

	For each individual component function $f_i(\bs{w})$, we have 
	\begin{align*}
	\norm{\bs{w_{t+1}} -  \wg}^2 =& \norm{\bs{w_t} -\wg}^2+ \eta_t^2\norm{\nabla f_{i}( \bs{w_{t}})}^2 - 2  \eta_t  \langle{ \bs{w_t} - \wg, \nabla f_{i}(\bs{w_t})\rangle}\\
	\leq    &\norm{\bs{w_t} -\wg}^2 +2\eta_t^2\| \nabla f_{i}(\bs{w_t})-\nabla f_{i}(\wg)\|^2+2\eta_t^2\norm{\nabla f_{i}( \wg)}^2 - 2  \eta_t  \langle{ \bs{w_t} - \wg, \nabla f_{i}(\bs{w_t})\rangle}\\
	\leq    &\norm{\bs{w_t} -\wg}^2 +2\eta_t^2L_i\langle{ \bs{w_t} -\wg ,\nabla f_{i}(\bs{w_t})-\nabla f_{i}(\wg)\rangle}+2\eta_t^2\norm{\nabla f_{i}( \wg)}^2\\
	&- 2  \eta_t  \langle{ \bs{w_t} -  \wg, \nabla f_{i}(\bs{w_t})\rangle} \\
	= &\norm{\bs{w_t} -\wg}^2 - 2\eta_t(1-\eta_t\sup_i L_i)\langle{\bs{w_t} - \wg, \nabla f_{i}(\bs{w_t})-\nabla f_{i}(\wg)\rangle}+2\eta_t^2\norm{\nabla f_{i}( \wg)}^2\\
	& -2  \eta_t  \langle{ \bs{w_t} -  \wg, \nabla f_{i}(\wg)\rangle}
	\end{align*}
	We next take an expectation with respect to the choice of $i$ conditional on $\bs{w_t}$ 
	\begin{align}
	\mathbb{E}_i\left[ \norm{\bs{w_{t+1}} -  \wg}^2 |\bs{w_{t}}  \right]
	\leq &\norm{\bs{w_t} -\wg}^2 -  2  \eta_t (1-\eta_t \sup_iL_i)\underbrace{\bigg\langle{ \bs{w_t} - \wg, \sum_{i\in \mathcal{G}} p_i(\bs{w_t})\left(\nabla f_{i}(\bs{w_t}) - \nabla f_{i}(\wg) \right)\bigg\rangle}}_{Term1} \notag\\
	&- 2  \eta_t  \langle{ \bs{w_t} - \wg,\sum_{i\in \mathcal{G}} p_i(\bs{w_t})\nabla f_{i}( \wg)\rangle}+ 2\eta_t^2\sum_{i\in \mathcal{G}} p_i(\bs{w_t}) \norm{\nabla f_{i}( \wg)}^2 \notag
	\\
	& 
	+ \eta_t^2\sum_{i\in \mathbb{O}}p_i(\bs{w_t})\norm{\nabla f_{i}( \bs{w_{t}})}^2+ 2  \eta_t  \underbrace{\langle{\wg- \bs{w_t},\sum_{i\in \mathbb{O}} p_i(\bs{w_t})\nabla f_{i}(\bs{w_t})\rangle}}_{Term2} \label{eq:contraction}
	\end{align}
	Now we first bound $Term1$ as follows
	\begin{align*}
	Term1 &\leq \min_{i}p_i(\bs{w_t}) \sum_{ i\in \mathcal{G}}\bigg\langle{ \bs{w_t} - \wg,\nabla f_{i}(\bs{w_t}) - \nabla f_{i}(\wg) \bigg\rangle} \\
	&\leq\min_{i  \in \mathcal{G}}p_i(\bs{w_t})   \lambda_{good}\norm{\bs{w_{t+1}} -  \wg}^2
	\end{align*}
	For $Term2$ we apply the property of the convex function $\langle{ \nabla f_i(\bs{v}), \bs{w}-\bs{v} \rangle}\leq f_i(\bs{w})-f_i(\bs{v})$
	$$ Term2\leq \sum_{i\in \mathbb{O}} p_i(\bs{w_t})\left(f_{i}(\wg)-f_{i}(\bs{w_{t}})\right)$$
	Putting the upper bound of $Term1$ and $Term2$ back to \eqref{eq:contraction} gives
	%
	\begin{align}
	\mathbb{E}_i\left[ \norm{\bs{w_{t+1}} - \wg}^2|\bs{w_{t}} \right]
	\leq & \left( 1-  2\eta_t  \lambda_{good }(1-\eta_t \sup_iL_i)\min_{ i\in \mathcal{G}}p_i(\bs{w_t}) \right)\norm{\bs{w_t} -\wg}^2+ R_t 
	\end{align}
	where 
	\begin{align*}
	R_t 
	=	& - 2  \eta_t \sum_{i\in \mathcal{G}} p_i(\bs{w_t}) \langle{ \bs{w_t} - \wg,\nabla f_{i}( \wg)\rangle}\\
	& + 2\eta_t^2\sum_{i\in \mathcal{G}} p_i(\bs{w_t}) \norm{\nabla f_{i}( \wg)}^2 + \eta_t^2\sum_{i\in \mathbb{O}}p_i(\bs{w_t})\norm{\nabla f_{i}( \bs{w_{t}})}^2+2  \eta_t \sum_{i\in \mathbb{O}}p_i(\bs{w_t}) \left(f_{i}(\wg)- f_i(\bs{w_t} )\right)
	\end{align*}	
	
\end{proof}
We have the following corollary that for noiseless setting, if we can  have some  good  initialization,  MKL-SGD is always better than SGD even the corrupted data is greater than half. For noisy setting, we can also perform better than SGD with  one more condition:  the noise is not large than the distance $\|\Delta_t\|^2$. This condition is not mild in the sense that $\|\bs{w}_t-\wg\|^2$ is always greater than $\|\wsgd -\wg\|^2$ for SGD algorithm and $\|\wsm -\wg\|^2$ for  MKL-SGD.

\begin{corollary}
	Suppose we have $\abs{\mathcal{G}}\leq \frac{n}{2}$. At iteration $t$ for $\eta_t\leq \frac{1}{\sup_i L_i}$,  the parameter $\bs{w_t}$  satisfies $\sup_{i \in \mathcal{G}}f_{i}(\bs{w_t})\leq \inf_{j \in \mathbb{O} }f_{j}(\bs{w_t}) $. Moreover, assume the noise level at optimal $\wg$ satisfies
	\begin{align}
	\text{either} \quad   \norm {\nabla f_{i}( \wg)} & \leq \frac{\lambda_{good}(1-\eta_t\sup_i L_i)/n}{1+ \sqrt{1+\eta_t(1-\eta_t\sup_i L_i )\lambda_{good}/n}}\|\bs{w_{t}} - \wg \|  \label{eq:noise-level1}, \text{ for } i \in \mathcal{G}   \\
	\text{or} \quad  \sum_{i\in\mathcal{G}  }\norm {\nabla f_{i}( \wg)} ^2& \leq \left( \frac{\lambda_{good}(1-\eta_t\sup_i L_i)\abs{\mathcal{G}}/n}{\sqrt{n}+ \sqrt{\sqrt{n}+\eta_t(1-\eta_t\sup_i L_i )\lambda_{good}\abs{\mathcal{G}}/n}} \right)^2\|\bs{w_{t}} - \wg \| ^2 \label{eq:noise-level2}.
	\end{align}
	Using the same setup, the vanilla SGD and MKL-SGD (K=2) algorithms yield respectively
	\begin{align*}
	\textbf{SGD} \quad \mathbb{E}_i\left[ \norm{\bs{w_{t+1}} - \wg}^2|\bs{w_{t}} \right]
	\leq & \left( 1-  2\eta_t  \lambda_{good }(1-\eta_t \sup_iL_i)\frac{\abs{\mathcal{G}}}{n}\right)\norm{\bs{w_t} -\wg}^2+  R_t^{(SGD)} \\
	\textbf{MKL-2} \quad\mathbb{E}_i\left[ \norm{\bs{w_{t+1}} - \wg}^2|\bs{w_{t}} \right]
	\leq & \left( 1-  2\eta_t  \lambda_{good }(1-\eta_t \sup_iL_i)\frac{\abs{\mathcal{G}}}{n}\right)\norm{\bs{w_t} -\wg}^2+R_t^{(MKL_2)} 
	\end{align*}
	where 
	\begin{align*}
	R_t^{(MKL_2)} &\leq R_t^{(SGD)}.
	\end{align*}
\end{corollary} 
%



\begin{proof}
	
	Start from the inequality \ref{eq:contraction} in the proof of Theorem  \ref{lem:R}.
	We have  $Term1$ as follows:
	\begin{align*}
	Term1
	= & \frac{\abs{\mathcal{G}}}{n}\bigg\langle{ \bs{w_t} - \wg, \sum_{i\in \mathcal{G}} \frac{p_i(\bs{w_t})}{{\abs{\mathcal{G}}}/{n}}\left(\nabla f_{i}(\bs{w_t}) - \nabla f_{i}(\wg) \right)\bigg\rangle}\\
	= & \frac{\abs{\mathcal{G}}}{n}\langle{ \bs{w_t} -\wg,  \nabla {F}_{good}(\bs{w_t})-\nabla {F}_{good}(\wg) \rangle}\\
	&+\frac{\abs{\mathcal{G}}}{n}\sum_{i\in \mathcal{G}} \left(\frac{p_i(\bs{w_t})}{{\abs{\mathcal{G}}}/{n}}-\frac{1}{ \abs{\mathcal{G}}}\right)\langle{ \bs{w_t} - \wg, \nabla f_{i}(\bs{w_t}) - \nabla f_{i}(\wg) \rangle}\\
	\geq& \lambda_{good} \frac{\abs{\mathcal{G}}}{n}\|\bs{w_t} - \wg\|^2+\sum_{i\in \mathcal{G}} \left(p_i(\bs{w_t})-\frac{1}{ n}\right)\langle{ \bs{w_t} - \wg, \nabla f_{i}(\bs{w_t}) - \nabla f_{i}(\wg) \rangle}
	\end{align*}
	
	
	Putting the terms back to \eqref{eq:contraction}, we have for $\eta_t \leq {1}/(\sup_i{L}_{i} )$
	\begin{align}
	\mathbb{E}_i\left[ \norm{\bs{w_{t+1}} - \wg}^2|\bs{w_{t}} \right]
	\leq & \left( 1-  2\eta_t  \lambda_{good }(1-\eta_t \sup_iL_i)\frac{\abs{\mathcal{G}}}{n}\right)\norm{\bs{w_t} -\wg}^2+ R_t 
	\end{align}
	where 
	\begin{align*}
	R_t 
	= & -  2\eta_t(1-\eta_t \sup_iL_i))\sum_{i\in \mathcal{G}} \left(p_i(\bs{w_t})-\frac{1}{n}\right)\langle{ \bs{w_t} - \wg, \nabla f_{i}(\bs{w_t}) - \nabla f_{i}(\wg) \rangle}\\
	& 
	- 2  \eta_t \sum_{i\in \mathcal{G}} p_i(\bs{w_t}) \langle{ \bs{w_t} - \wg,\nabla f_{i}( \wg)\rangle}\\
	& 
	+ 2\eta_t^2\sum_{i\in \mathcal{G}} p_i(\bs{w_t}) \norm{\nabla f_{i}( \wg)}^2 + \eta_t^2\sum_{i\in \mathbb{O}}p_i(\bs{w_t})\norm{\nabla f_{i}( \bs{w_{t}})}^2+2  \eta_t \sum_{i\in \mathbb{O}}p_i(\bs{w_t}) \left(f_{i}(\wg)- f_i(\bs{w_t} )\right)
	\end{align*}
	Now we analyse the term $R_t$ for vanilla SGD and M$K$L-SGD($K=2$) respectively.
	For vanilla SGD, we have $p_i(\bs{w_t}) =\frac{1}{n}$ and $\sum_{i\in \mathcal{G}}\nabla f_{i}( \wg) =0$, which results in 
	\begin{align*}
	R_t^{(SGD)}=  \frac{2\eta_t^2}{n}\sum_{i\in \mathcal{G}}\norm{\nabla f_{i}( \wg)}^2 + \frac{\eta_t^2}{n}\sum_{i\in \mathbb{O}}\norm{\nabla f_{i}( \bs{w_{t}})}^2+\frac{2\eta_t}{n} \sum_{i\in \mathbb{O}} \left(f_{i}(\wg)- f_i(\bs{w_t} )\right)
	\end{align*}
	Note that M$K$L-SGD for $K=2$ have 
	\begin{align}
	p_{m_i(\bs{w})}(\bs{w}) = \dfrac{2(n-i)}{n(n-1)} 
	\end{align}
	where $m_1(\bs{w}), m_2(\bs{w}), m_3(\bs{w}), \dots m_n(\bs{w})$ are the indices of data samples for some $\bs{w}$:
	\begin{align*}
	f_{m_1(\bs{w})}(\bs{w}) \leq 	f_{m_2(\bs{w})}(\bs{w}) \leq \dots \leq 	f_{m_n(\bs{w})}(\bs{w}) 
	\end{align*}
	Suppose the iteration $\bs{w_t}$ satisfies that $f_{i}(\bs{w_t})<f_{j}(\bs{w_t})$ for $i \in \mathcal{G}, j \in \mathbb{O} $.
	For $\abs{\mathcal{G}}\leq \frac{n}{2}$, we have for 
	\begin{align*}
	R_t^{(MKL_2)}=&-  2\eta_t(1-\eta_t \sup_iL_i)\sum_{i=1}^{\abs{\mathcal{G}}} \frac{\left(n-2i+1\right)}{n(n-1)}\langle{ \bs{w_t} - \wg, \nabla f_{m_i}(\bs{w_t}) - \nabla f_{m_i}(\wg) \rangle}\\
	& 
	+ 2  \eta_t \sum_{i=1}^{\abs{\mathcal{G}}} \dfrac{2(n-i)}{n(n-1)} \left( \langle{  \wg-\bs{w_t},\nabla f_{i}( \wg)\rangle}+ \eta_t \norm {\nabla f_{i}( \wg)}^2\right) \\
	& 
	+ \eta_t^2\sum_{i=\abs{\mathcal{G}}+1 }^n\dfrac{2(n-i)}{n(n-1)}\norm{\nabla f_{i}( \bs{w_{t}})}^2+ 2  \eta_t \sum_{i=\abs{\mathcal{G}}+1 }^n\dfrac{2(n-i)}{n(n-1)} \left(f_{i}(\wg)- f_i(\bs{w_t} )\right)\\
	\leq &-  2\eta_t(1-\eta_t \sup_iL_i)\frac{ \abs{\mathcal{G}}\lambda_{good}}{n(n-1)} \norm{\bs{w_t} - \wg}^2 \\
	&+ \frac{4\eta_t}{n}   \sum_{i=1}^{\abs{\mathcal{G}}}\left(\|  \wg-\bs{w_t}\|\|\nabla f_{i}( \wg)\|+ \eta_t \norm {\nabla f_{i}( \wg)}^2\right) \\
	& 
	+ \sum_{i=\abs{\mathcal{G}}+1 }^n\dfrac{\eta_t^2}{n}\norm{\nabla f_{i}( \bs{w_{t}})}^2+  \sum_{i=\abs{\mathcal{G}}+1 }^n\dfrac{2  \eta_t}{n} \left(f_{i}(\wg)- f_i(\bs{w_t} )\right)
	\end{align*}
	We will have $R_t^{(MKL_2)} \leq  R_t^{(SGD)}$ if  the following inequality holds
	\begin{align}
	(1-\eta_t \sup_iL_i)\frac{ \abs{\mathcal{G}}\lambda_{good}}{(n-1)} \norm{\bs{w_t} - \wg}^2  \geq    \sum_{i=1}^{\abs{\mathcal{G}}}\left(2\|  \wg-\bs{w_t}\|\|\nabla f_{i}( \wg)\|+\eta_t \norm {\nabla f_{i}( \wg)}^2\right). \label{eq:get}    
	\end{align}
	Indeed, for the noise level $\norm {\nabla f_{i}( \wg)}^2$ satisfying \eqref{eq:noise-level1} we have for  $i \in \mathcal{G}$,
	\begin{align*}
	(1-\eta_t \sup_iL_i)\frac{ \lambda_{good}}{(n-1)} \norm{\bs{w_t} - \wg}^2 
	&\geq 2\|  \wg-\bs{w_t}\|\|\nabla f_{i}( \wg)\|+\eta_t \norm {\nabla f_{i}( \wg)}^2. 
	\end{align*}
	Summing up the terms in $i\in \mathcal{G}$, we get \eqref{eq:get}.
	For the noise level $\norm {\nabla f_{i}( \wg)}^2$ satisfying \eqref{eq:noise-level2} we have  
	\begin{align*}
	(1-\eta_t \sup_iL_i)\frac{ \lambda_{good}\abs{\mathcal{G}} }{(n-1)} \norm{\bs{w_t} - \wg}^2 
	&\geq \left(2\|  \wg-\bs{w_t}\| \sqrt{n\sum_{i\in \mathcal{G}} \|\nabla f_{i}( \wg)\|^2}+\eta_t  \sum_{i\in \mathcal{G}}\norm {\nabla f_{i}( \wg)}^2\right)\\
	&\geq 2\|  \wg-\bs{w_t}\|\sum_{i\in \mathcal{G}}\|\nabla f_{i}( \wg)\|+\eta_t\sum_{i\in \mathcal{G}} \norm {\nabla f_{i}( \wg)}^2. 
	\end{align*}
	which results in \eqref{eq:get}.
\end{proof}


\newpage

\subsection{More experimental results}

\subsubsection{Linear Regression}\label{sec:linReg}

Here, we show that there exists a tradeoff for \mkl~ between the rate of convergence and robustness the algorithm provides against outliers depending on the value of the parameter $k$. Larger the $k$, more robust is the algorithm, but slower is the rate of convergence. The algorithm outperforms median loss SGD and SGD. We also experimentd with other order statistics and observed that for most general settings \mkl~ was the best to pick. Note that the outliers are chosen from $\mathcal{N}(0,1)$ distribution. 
\begin{figure}[!h]
	\subfloat[][k=2]{\includegraphics[scale=0.35]{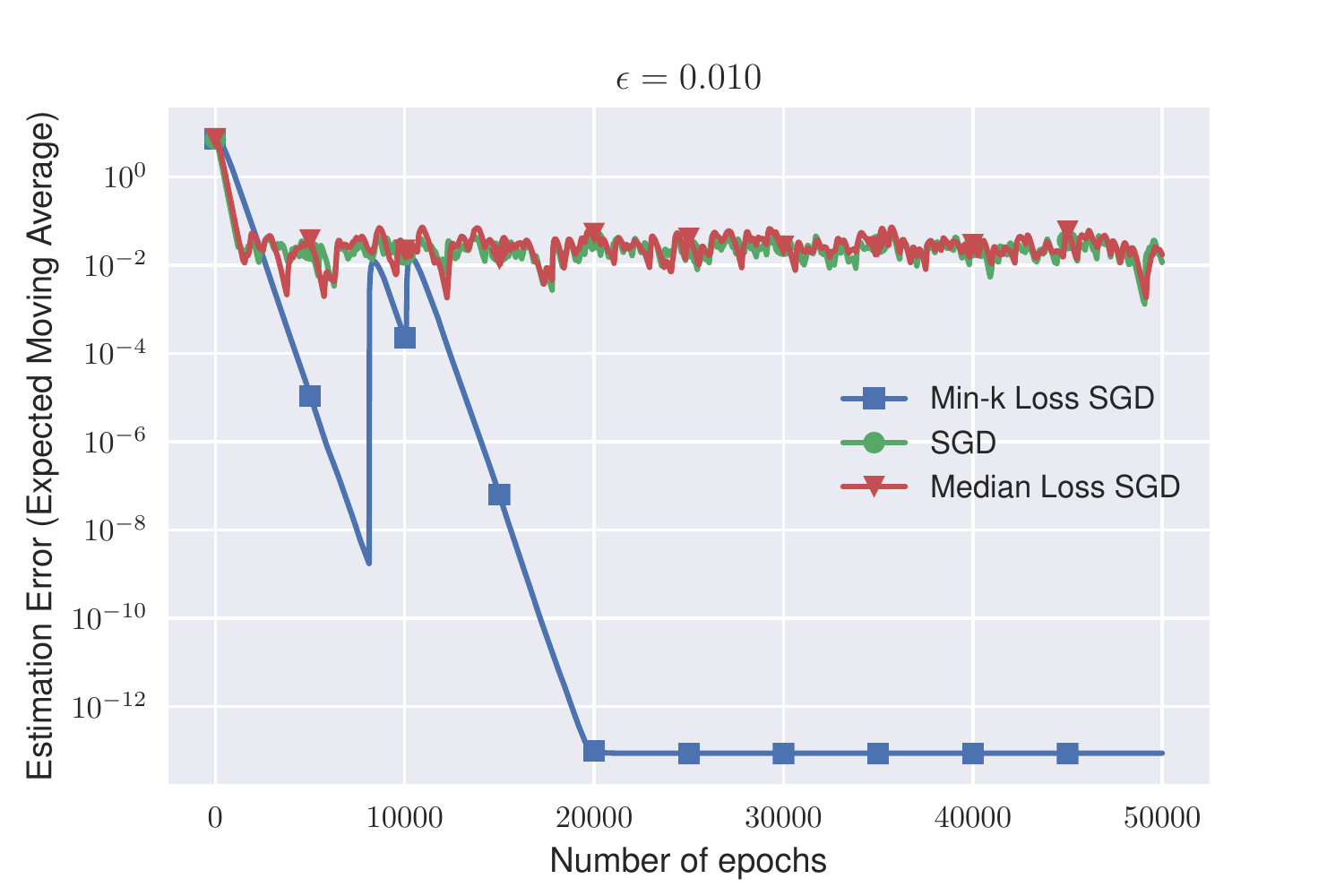}}
	\subfloat[][k=3]{\includegraphics[scale=0.35]{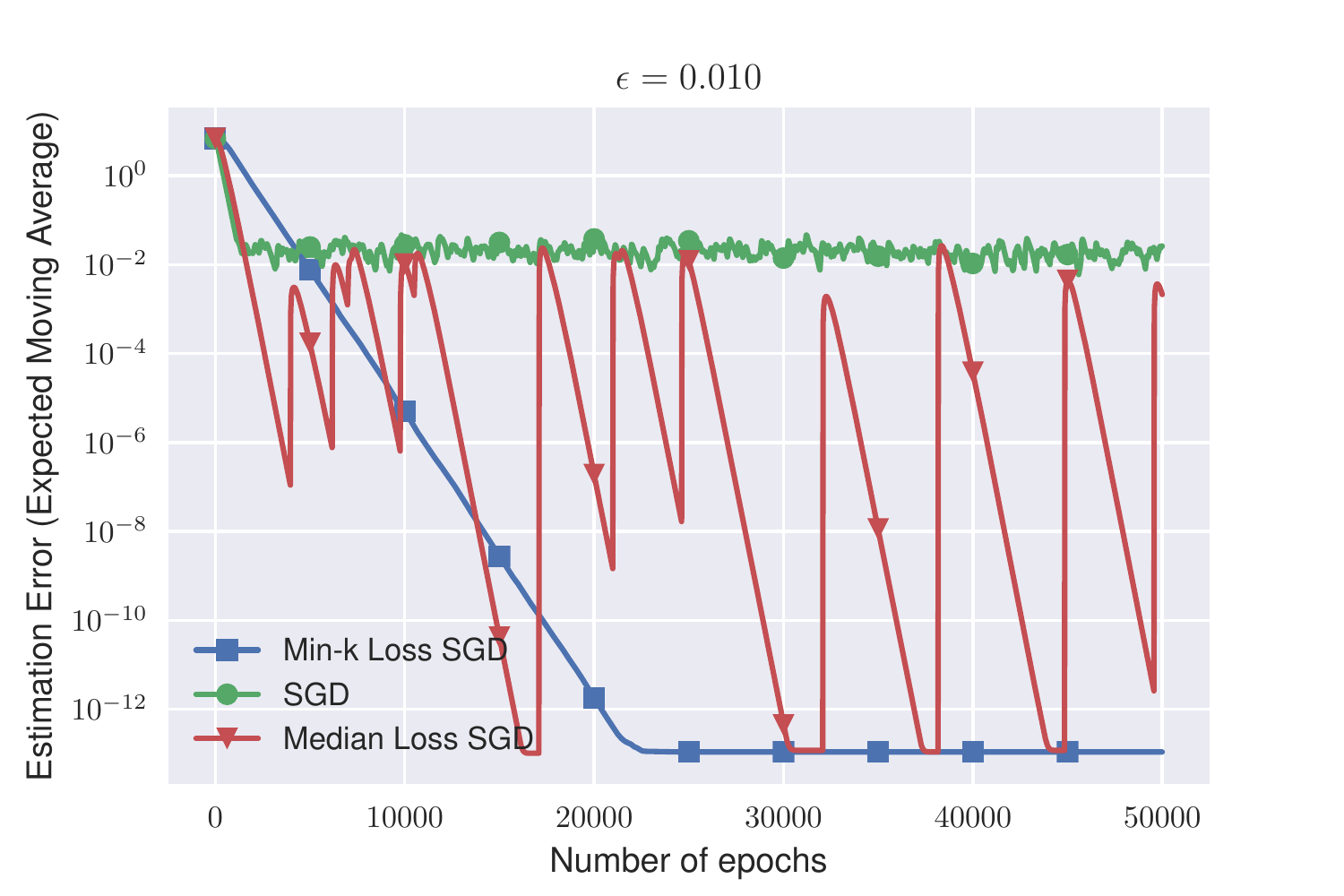}}
	\subfloat[][k=5]{\includegraphics[scale=0.35]{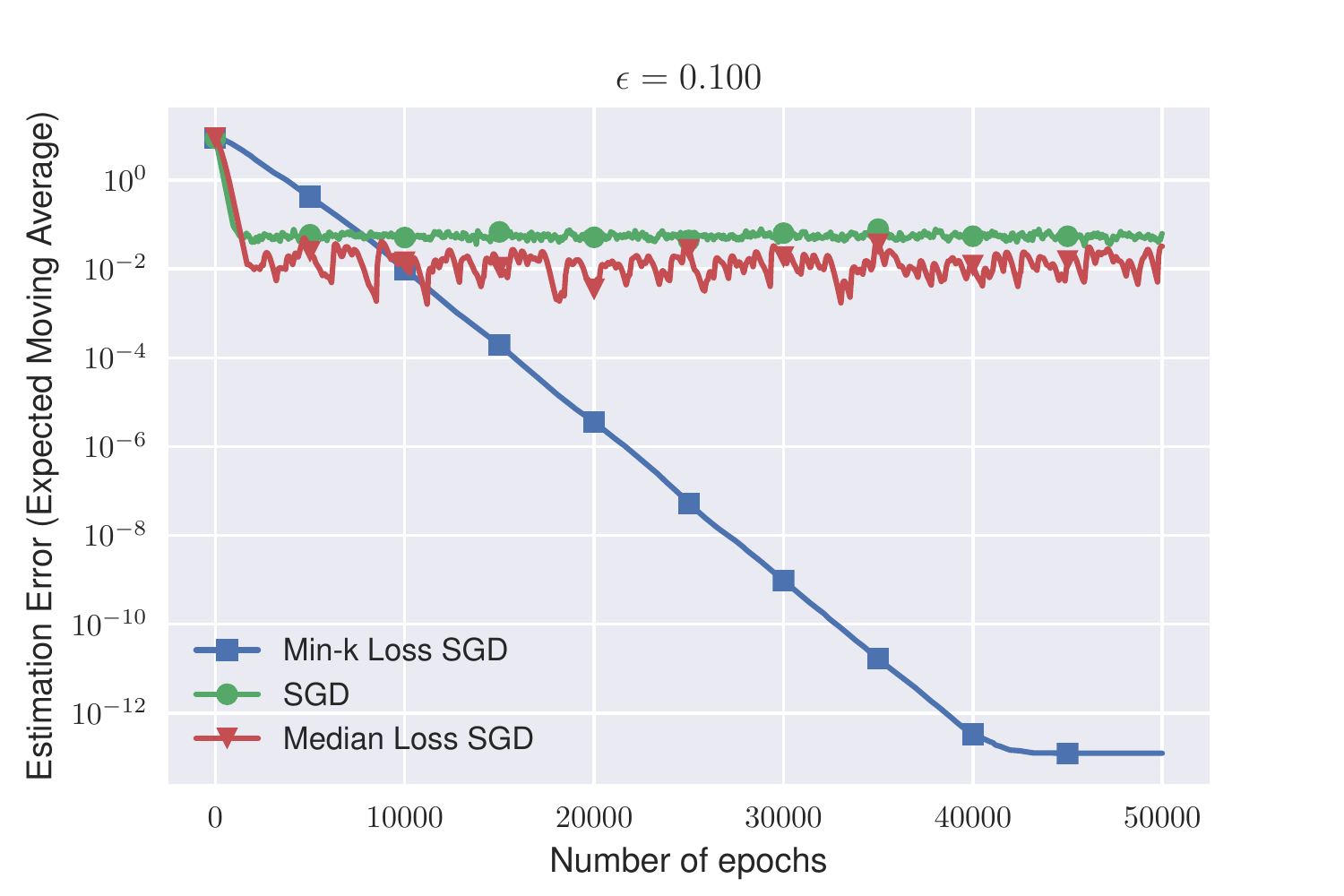}}
	\caption{Comparing the performance of \mkl, SGD and Median loss SGD in the noiseless setting, $d=50$. }
\end{figure}

\begin{figure}[!h]
	\subfloat[][k=3]{\includegraphics[scale=0.35]{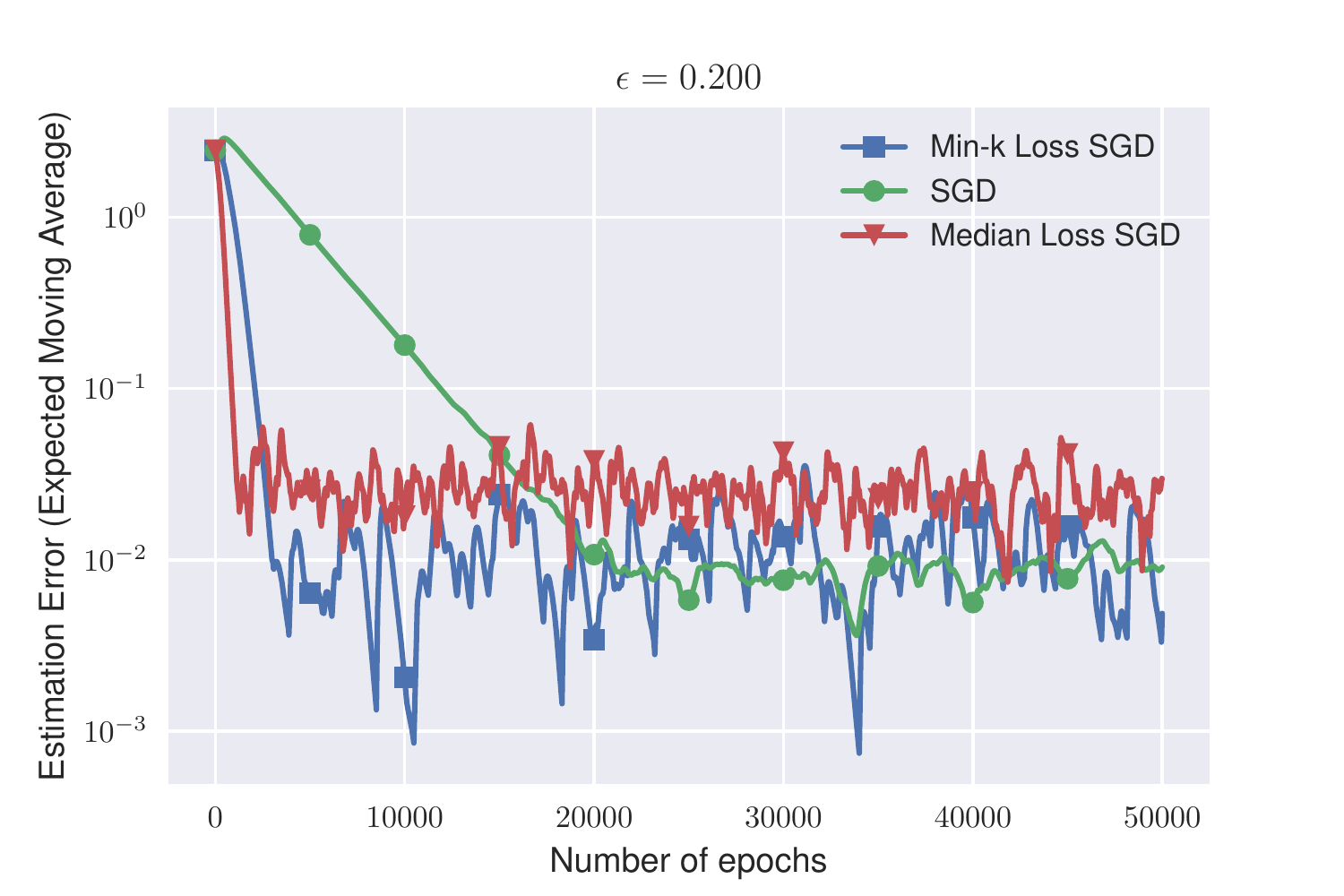}}
	\subfloat[][k=5]{\includegraphics[scale=0.35]{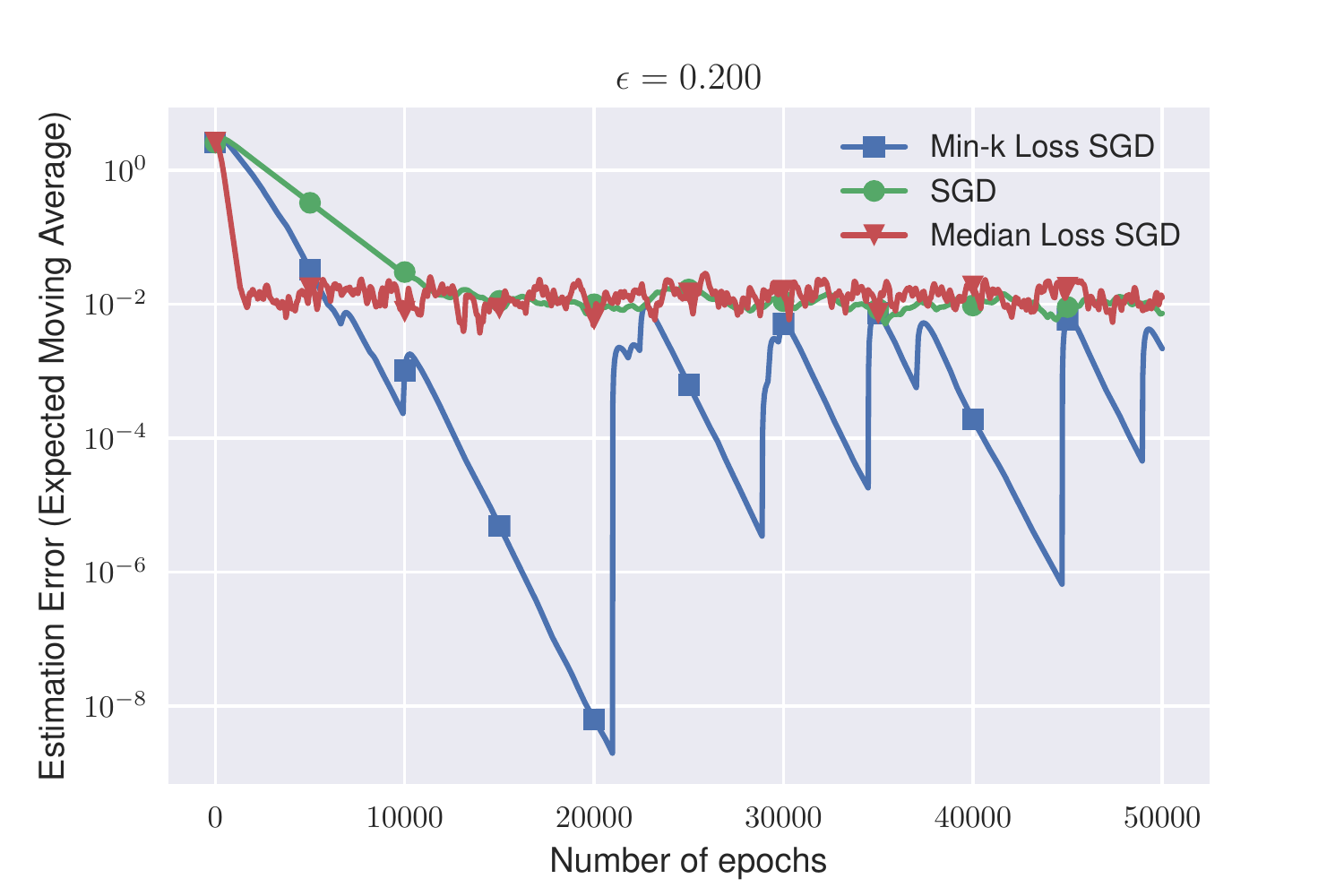}}
	\subfloat[][k=9]{\includegraphics[scale=0.35]{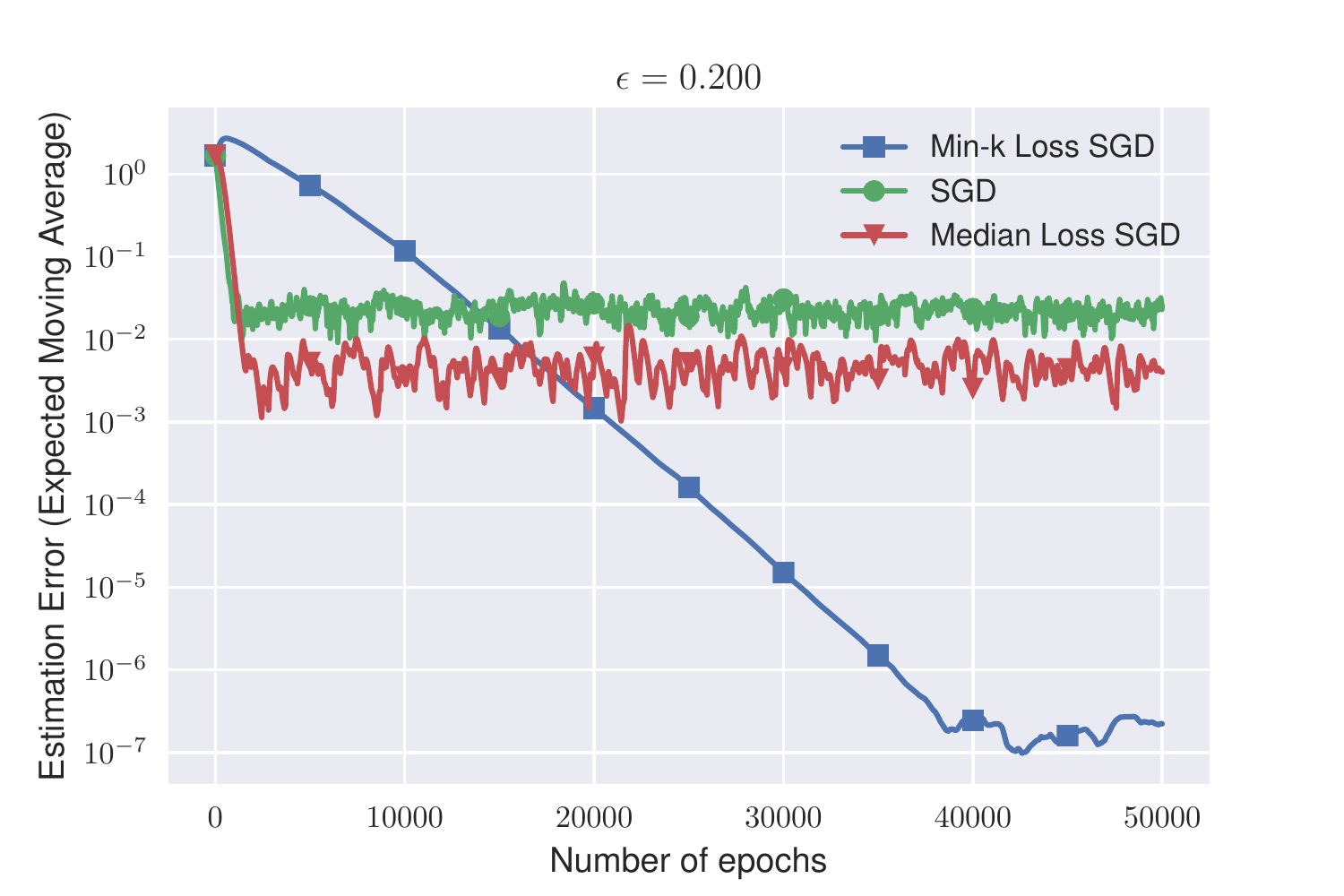}}
	\caption{Comparing the performance of \mkl, SGD and Median loss SGD in the noisy setting, $d=10$, Noise variance=0.0001}
\end{figure}

\begin{figure}[!h]
	\centering
	\subfloat[][k=3]{\includegraphics[scale=0.35]{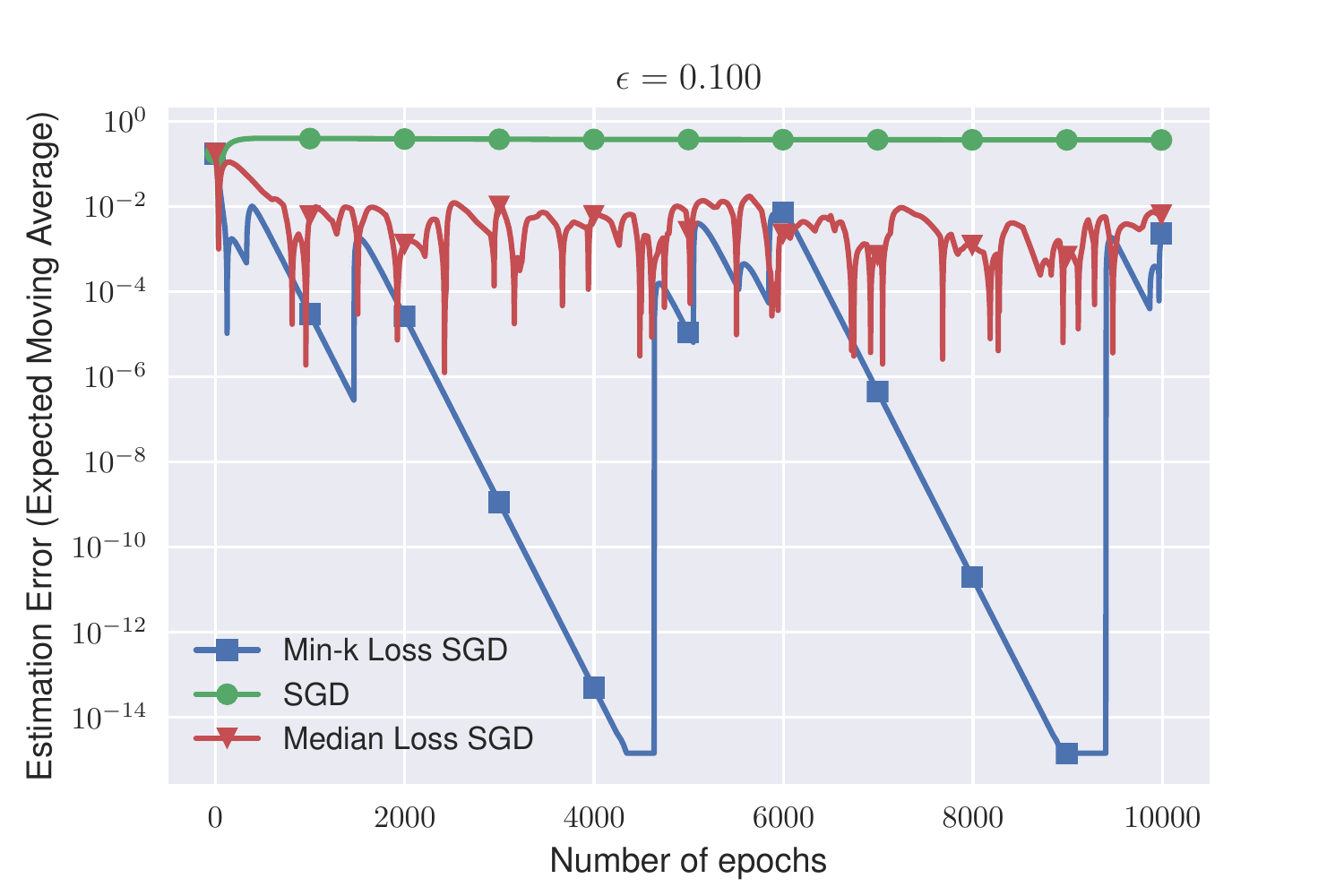}}
	\subfloat[][k=5]{\includegraphics[scale=0.35]{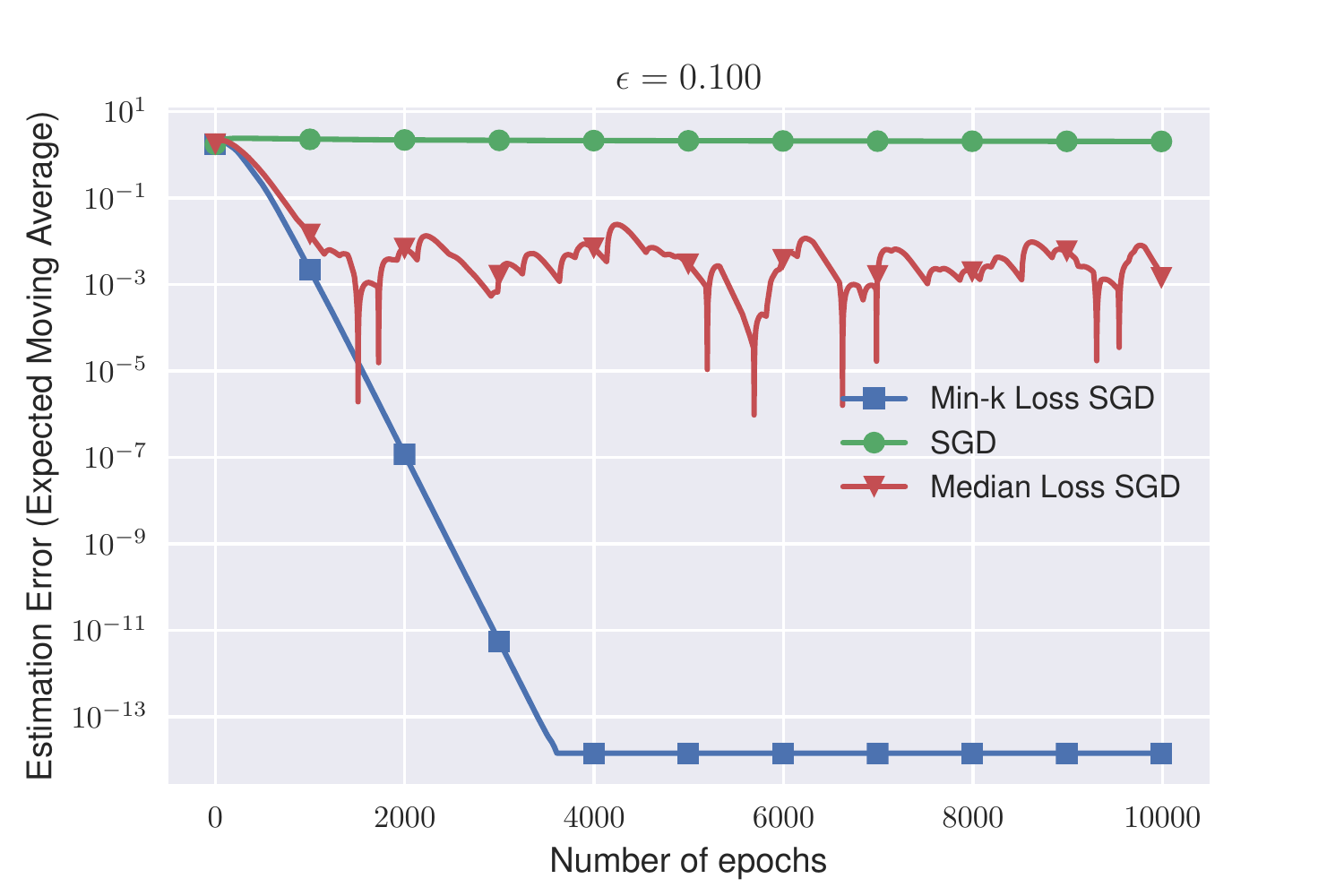}}
	\caption{Comparing the performance of \mkl, SGD and Median loss SGD in the noiseless setting, $d=25$, Noise variance=0.01}
\end{figure}



\newpage
\subsubsection{Neural Network Experiments}

Here, we show that in presence of outliers instead of tuning other hyperparameters like learning rate, tuning over $k$ might lead to significant gains in performances for deep neural networks. To illustrate this we play around with two commonly used noise models: random noise and directed noise. In the random noise model, the outlier label is randomly assigned while for the directed noise model for some class `a', the outlier is assigned the same label `b', similarly all the outliers for class `b' are assigned label `c' and so on.

\begin{table}[!ht]
	\centering
	\begin{tabular}{|c|c|c|c|c|c|c|c|} 
		\hline
		\textbf{Dataset}  &\multicolumn{7}{c|}{\textbf{MNIST with 2-layer CNN (Directed Noise)}}  \\ 	\hline
		Optimizer &\textbf{SGD} &\multicolumn{5}{c|}{\textbf{\mkl}}  &\textbf{Oracle}\\ \hline 
		\backslashbox{$\epsilon$}{$\alpha$} &$ 1.0$ &$ 0.9$ &$0.8$ &$0.7$ &$0.6$  &$0.5$ &1.0\\ \hline
		$0.1$ &96.76  &97.23 &95.89 &\textbf{97.47} &96.34  &94.54 &98.52 \\  \hline
		$0.2$ &92.54 &95.81 &95.58 &\textbf{97.46}  &97.03 &95.76 &98.33\\ \hline
		$0.3$ &85.77 &91.56 &93.59 &95.30 &\textbf{96.54} &95.96 &98.16 \\ \hline
		$0.4$ &71.95 &78.68 &82.25 &85.93 &91.29 &\textbf{94.20} &97.98\\ \hline
	\end{tabular}
	\caption{In this experiments, we train a standard 2 layer CNN on subsampled MNIST ($5000$ training samples with labels corrupted using random label noise). We train over $80$ epochs using an initial learning rate of $0.05$ with the decaying schedule of factor $5$ after every $30$ epochs. The reported accuracy is based on the true validation set. The results of the MNIST dataset are reported as the mean of $5$ runs. For the \mkl~algorithm, we introduce a more practical variant that evaluates $k$ sample losses and picks a batch of size $\alpha  k$ where $k=10$.}
\end{table}

\begin{table}[!h]
	\centering
	\begin{tabular}{|c|c|c|c|c|c|c|c|} 
		\hline
		\textbf{Dataset}  &\multicolumn{7}{c|}{\textbf{MNIST with 2-layer CNN (Random Noise)}}  \\ 	\hline
		Optimizer &\textbf{SGD} &\multicolumn{5}{c|}{\textbf{\mkl}} &\textbf{Oracle} \\ \hline 
		\diagbox{$\epsilon$}{$\alpha$} &$ 1.0$ &$ 0.9$ &$0.8$ &$0.7$ &$0.6$  &$0.5$ &1.0 \\ \hline
		$0.1$  &96.91 &97.9    &\textbf{98.06} &97.59  &96.49  &94.43 &98.44\\  \hline
		$0.2$ &93.94 &95.5    &96.16 &97.02   &\textbf{97.04}  &96.25 &98.18\\ \hline
		$0.3$ &87.14 &90.71   &91.60 &92.97  &94.54  &\textbf{95.36} &97.8\\ \hline
		$0.4$ &71.83 &74.31   &76.6   &78.30  & 77.58  &\textbf{80.86} &97.16\\
		\hline
	\end{tabular}
	\caption{In this experiments, we train a standard 2 layer CNN on subsampled MNIST ($5000$ training samples with labels corrupted using random label noise). We train over $80$ epochs using  an initial learning rate of $0.05$ with the decaying schedule of factor $5$ after every $30$ epochs. The reported accuracy is based on the true validation set. The results of the MNIST dataset are reported as the mean of $5$ runs. For the \mkl~algorithm, we introduce a more practical variant that evaluates $k$ sample losses and picks a batch of size $\alpha  k$ where $k=10$.}
\end{table}

\begin{table}[!h]
	\centering
	\begin{tabular}{|c|c|c|c|c|c|c|c|c|} 
		\hline
		\textbf{Dataset}  &\multicolumn{7}{c|}{\textbf{CIFAR-10 with Resnet-18 (Directed Noise)}}  \\ 	\hline
		Optimizer &\textbf{SGD} &\multicolumn{5}{c|}{\textbf{\mkl}}  &\textbf{Oracle}\\ \hline 
		\backslashbox{$\epsilon$}{$\alpha$} &$ 1.0$ &$ 0.9$ &$0.8$ &$0.7$ &$0.6$ &$0.5$  &1.0\\ \hline
		$0.1$ &79.1  &77.52 &79.57  &81.00 &\textbf{81.94} &80.53 &84.56 \\  \hline
		$0.2$ &72.29 &69.58 &70.17 &72.76 &77.77 &\textbf{78.93} &84.40\\ \hline
		$0.3$ &63.96 &61.43 &60.46 &61.58 &66.49 &\textbf{69.57} &84.66\\ \hline
		$0.4$ &\textbf{52.4} &51.53  &51.04 &51.07 &53.57 &51.2 &84.42 \\ \hline
	\end{tabular}
	\caption{In this experiments, we train Resnet 18 on CIFAR-10 ($50000$ training samples with labels corrupted using directed label noise). We train over $200$ epochs using an initial learning rate of $0.05$ with the decaying schedule of factor $5$ after every $90$ epochs. The reported accuracy is based on the true validation set. The results of the CIFAR-10 dataset are reported as the mean of $3$ runs. For the \mkl~algorithm, we introduce a more practical variant that evaluates $k$ sample losses and picks a batch of size $\alpha  k$ where $k=16$.}
\end{table}

%
%
%

%

\end{document}